\documentclass[twoside,11pt]{article}

\usepackage[preprint]{jmlr2e}

\usepackage{amsfonts,amsmath,amssymb}
\usepackage{physics,mathtools,bm,mathrsfs,mleftright}
\usepackage{xcolor,enumerate,autobreak}
\allowdisplaybreaks[4]
\usepackage{wrapfig,subfigure,multirow}
\usepackage{hyperref,prettyref}
\usepackage{natbib}

\usepackage{tikz}

\newtheorem{assumption}{Assumption}

\newrefformat{eq}{\textup{(\ref{#1})}}
\newrefformat{lem}{Lemma~\textup{\ref{#1}}}
\newrefformat{thm}{Theorem~\textup{\ref{#1}}}
\newrefformat{cor}{Corollary~\textup{\ref{#1}}}
\newrefformat{prop}{Proposition~\textup{\ref{#1}}}
\newrefformat{assu}{Assumption~\textup{\ref{#1}}}
\newrefformat{remark}{Remark~\textup{\ref{#1}}}
\newrefformat{example}{Example~\textup{\ref{#1}}}
\newrefformat{def}{Definition~\textup{\ref{#1}}}
\newrefformat{sec}{Section~\textup{\ref{#1}}}
\newrefformat{subsec}{Subsection~\textup{\ref{#1}}}
\newrefformat{subsubsec}{Subsection~\textup{\ref{#1}}}

\newcommand{\ep}{\varepsilon}
\newcommand{\R}{\mathbb{R}}
\newcommand{\E}{\mathbb{E}}

\newcommand{\caH}{\mathcal{H}}

\newcommand{\caL}{\mathcal{L}}

\newcommand{\caX}{\mathcal{X}}

\newcommand{\bbN}{\mathbb{N}}
\newcommand{\bbP}{\mathbb{P}}

\newcommand{\bbT}{\mathbb{T}}

\newcommand{\T}{\top}

\newcommand{\hf}{\frac{1}{2}}

\newcommand{\mr}{\mathrm}
\providecommand{\ang}[1]{\left\langle{#1}\right\rangle}

\newcommand{\xk}[1]{\left(#1\right)}
\newcommand{\zk}[1]{\left[#1\right]}
\newcommand{\dk}[1]{\left\{#1\right\}}
\newcommand{\xkm}[1]{\mleft(#1\mright)}

\providecommand{\cref}{\prettyref}

\providecommand{\cref}{\prettyref}

\usepackage{lastpage}
\jmlrheading{00}{0000}{1-\pageref{LastPage}}{0/00}{0/00}{00-0000}{Yicheng Li and Qian Lin}

\ShortHeadings{Diagonal adaptive kernel}{Li and Lin}
\firstpageno{1}

\begin{document}

\title{Diagonal Over-parameterization in Reproducing Kernel Hilbert Spaces as an Adaptive Feature Model: Generalization and Adaptivity}

\author{\name Yicheng Li \email liyc22@mails.tsinghua.edu.cn \\
       \addr Department of Statistics and Data Science\\
       Tsinghua University\\
       Beijing, 100084, China
       \AND
       \name Qian Lin \email qianlin@tsinghua.edu.cn \\
       \addr Department of Statistics and Data Science\\
       Tsinghua University\\
       Beijing, 100084, China}

\editor{My editor}

\maketitle

\begin{abstract}

This paper introduces a diagonal adaptive kernel model that dynamically learns kernel eigenvalues and output coefficients simultaneously during training.
Unlike fixed-kernel methods tied to the neural tangent kernel theory, the diagonal adaptive kernel model adapts to the structure of the truth function, significantly improving generalization over fixed-kernel methods, especially when the initial kernel is misaligned with the target.
Moreover, we show that the adaptivity comes from learning the right eigenvalues during training, showing a feature learning behavior.
By extending to deeper parameterization, we further show how extra depth enhances adaptability and generalization.
This study combines the insights from feature learning and implicit regularization
and provides new perspective into the adaptivity and generalization potential of neural networks beyond the kernel regime.
\end{abstract}


\section{Introduction}\label{sec:introduction}

Understanding the generalization capabilities of neural networks is a central problem in modern machine learning.
The remarkable success of neural networks across diverse domains, such as natural language processing and computer vision, has highlighted their ability to generalize well beyond their training data.
This phenomenon has sparked intense research interest, leading to the development of several theoretical frameworks over the past decade.
Among these, the Neural Tangent Kernel (NTK) theory has emerged as a foundational tool for analyzing the training dynamics of wide neural networks in the infinite-width regime~\citep{jacot2018_NeuralTangent}.
NTK theory demonstrates that under specific conditions, neural networks behave like kernel methods during training, offering valuable insights into their convergence and generalization properties.
Parallel to this, an extensive literature on kernel regression~(see, e.g., \citet{caponnetto2007_OptimalRates,yao2007_EarlyStopping}) has studied its generalization properties, showing its minimax optimality under certain conditions and providing insights into the bias-variance trade-off.
Thus, one can almost fully understand the generalization properties of neural networks in the NTK regime by analyzing the kernel regression method.

However, the application of NTK theory to analyze neural networks, while invaluable, essentially frames the problem within a traditional statistical method by a fixed kernel.
The NTK analysis, by its reliance on the fixed kernel approximation, can not entirely account for the adaptability and flexibility exhibited by neural networks, particularly those of finite width that deviate from the theoretical infinite-width limit~\citep{woodworth2020_KernelRich}.
However, empirical evidence and theoretical studies have shown that real-world neural networks exhibit dynamic changes in their feature representations during training.
This evolution, often termed \emph{feature learning}, has been increasingly recognized as a key driver behind the adaptability and success of neural networks beyond the NTK regime~\citep{wenger2023_DisconnectTheory,seleznova2022_AnalyzingFinite}.

Recent studies on understanding feature learning have particularly focused the context of shallow neural networks as a random feature model.
It has been found that that even a single gradient descent step can induce significant changes in the feature space, enabling the learning of task-relevant representations~\citep{ba2022_HighdimensionalAsymptotics}.
This phenomenon enables neural networks to achieve a level of generalization superior to that of fixed-kernel methods
~\citep{damian2022_NeuralNetworks,dandi2023_HowTwoLayer,moniri2024_TheoryNonLinear,cui2024_AsymptoticsFeature,bordelon2024_HowFeature}.
However, these studies are mostly restricted to scenarios where the feature weights are updated with a single gradient descent step, and the output coefficients are subsequently computed through a separate ridge regression.

Another branch of research has focused on the over-parameterization nature of neural networks beyond the NTK regime, exploring how over-parameterized models trained with gradient methods can lead to \textit{implicit regularization} and improve generalization.
Studies in this domain~(e.g., \citet{hoff2017_LassoFractional,gunasekar2017_ImplicitRegularization,arora2019_ImplicitRegularization,kolb2023_SmoothingEdges}) have revealed that over-parameterized models, particularly those trained with gradient descent and its variants, exhibit biases towards simpler, more generalizable functions, even in the absence of explicit regularization terms.
Additionally, recent works~\citep{vaskevicius2019_ImplicitRegularization,zhao2022_HighdimensionalLinear,li2021_ImplicitSparse} have shown that in the setting of high-dimensional linear regression, over-parameterized models with proper initialization and early stopping can achieve minimax optimal recovery under certain conditions.
These findings highlight the advantages of over-parameterized models in transcending traditional statistical paradigms.

Building on these lines of research,
we propose to investigate \emph{adaptive feature model}~\citep{zhang2024_StatisticalUnderstanding}, wherein the kernel, or equivalently the feature map,
evolves dynamically during training to better align with the underlying data structure.
Particularly, introducing a certain parameterization of the feature map, we can study the gradient descent on both the feature parameters and the output coefficients.
This paradigm has a two-fold interpretation:
in terms of feature learning, the parameterized feature map enables the learning of feature during training,
while in terms of implicit regularization, the extra parameters introduced in the feature map allow the model to implicitly regularize with the gradient descent dynamics.
Moreover, by considering only gradient training dynamics, we can obtain a clearer understanding of the adaptive nature during training.

In this paper, we study a \textit{diagonal adaptive feature} model (see \cref{eq:DiagonalFeatureMap} and \cref{eq:DiagonalKernelGD})
where the eigenvalues of the kernel and the coefficients are simultaneously learned during training,
which serve as a first concrete instantiation of the adaptive feature learning paradigm.
Under the non-parametric regression framework, we show that, by dynamically adjusting the eigenvalues of the kernel, the model can significantly improve its generalization properties compared with the fixed-kernel method.
Moreover, the adaptivity of the model comes from learning the right eigenvalues during the training process, showing a feature learning behavior.
These results provide insights into the adaptivity and generalization properties of neural networks beyond the NTK regime.


\subsection{Our Contributions}

This paper is an extended version of the conference paper \citet{li2024_ImprovingAdaptivity}.
In this paper, we study the generalization properties of a diagonal adaptive kernel regression model, where the eigenvalues of the kernel are learned simultaneously with the output coefficients during training.
Our contributions are summarized as follows:

\textit{Limitations of the (fixed) kernel regression.}
In this work, we first investigate the limitations of the fixed kernel regression method by specific examples,
illustrating that the traditional kernel regression method suffers from the misalignment between the kernel and the truth function.
We show that even when the eigen-basis of the kernel is fixed, the associated eigenvalues, particularly their alignment with the truth function's coefficients in the eigen-basis, can significantly affect the generalization properties of the method.

\textit{Advantages of diagonal adaptive kernel.}
Focusing on the alignment between the kernel's eigenvalues and the truth function's coefficients,
we introduce a diagonal parameterization of the kernel and study the gradient descent dynamics on both the feature parameters and the output coefficients.
Our theoretical results show the following:
\begin{itemize}
  \item We show in \cref{thm:EigenvalueGD} that the diagonal adaptive kernel method, with proper early-stopping, can achieve nearly the oracle convergence rate regardless of the underlying structure of the signal, significantly outperforming the vanilla fixed-kernel method when the misalignment is severe.
  \item In \cref{prop:EigLearn}, we show that the adaptive kernel can adapt to the signal's structure by learning the right eigenvalues during training, showing a feature learning behavior.
  \item Considering a deeper parameterization in \cref{thm:EigenvalueDeepGD}, we show that adding depth can further ease the impact of the initial choice of the eigenvalues, thus improving the generalization capability of the model.
  \item As the main technical improvement, we consider directly the non-parametric regression framework instead of the Gaussian sequence model approximation in the conference paper~\citep{li2024_ImprovingAdaptivity}.
\end{itemize}


\subsection{Related Works}

\textit{Fixed kernel regression.}
The regression problem with a fixed kernel has been well studied in the literature~\citep{andreaschristmann2008_SupportVector}.
With proper regularization, kernel methods can achieve the minimax optimal rates under various conditions
~\citep{caponnetto2007_OptimalRates,lin2018_OptimalRatesa,fischer2020_SobolevNorm,zhang2023_OptimalityMisspecifieda}.
More recently, a sequence of works provided more refined results on the generalization error of kernel methods
~\citep{bordelon2020_SpectrumDependent,cui2021_GeneralizationError,li2023_SaturationEffect,mallinar2022_BenignTempered,li2023_AsymptoticLearning,li2024_GeneralizationError}.

\textit{The NTK regime of neural networks.}
Over-parameterized neural networks are connected to kernel methods through the neural tangent kernel (NTK) theory proposed by \citet{jacot2018_NeuralTangent},
which shows that the dynamics of the neural network at infinite width limited can be approximated by a kernel method with respect to the corresponding NTK\@.
The theory was further developed by many follow-up works~\citep{arora2019_ExactComputation,arora2019_FinegrainedAnalysisa,du2018_GradientDescent,lee2019_WideNeurala,allen-zhu2019_ConvergenceTheory}.
Also, the properties on the corresponding NTK have also been studied~\citep{geifman2020_SimilarityLaplace,bietti2020_DeepEquals,li2024_EigenvalueDecay}.

\textit{Feature learning.}
A lot of attention has been paid to the feature learning behavior of neural networks.
Starting from \citet{ba2022_HighdimensionalAsymptotics},
most of these works focus on shallow neural networks as a random feature model where the feature weights are firstly trained by one-step gradient descent and then the output weights are computed using ridge regression, showing the generalization properties of the feature learning behavior~\citep{moniri2024_TheoryNonLinear,cui2024_AsymptoticsFeature,bordelon2024_HowFeature,lejeune2023_AdaptiveTangent,dandi2024_RandomMatrix,yang2022_FeatureLearning,damian2022_NeuralNetworks}.
These studies reveal that feature learning manifests as adaptations in the feature matrix, often leading to rank-one ``spikes'' in its spectrum, which align with the structure of the target function.
Beside the random feature model, the adaptive kernel perspective has also been investigated in various particular forms~\citep{woodworth2020_KernelRich,gatmiry2021_OptimizationAdaptive}.

\textit{Over-parameterization as implicit regularization.}
There has been a surge of interest in understanding the role of over-parameterization in deep learning.
One perspective is that over-parameterized models trained by gradient-based methods can expose certain implicit bias towards simple solutions,
which include linear models~\citep{hoff2017_LassoFractional},
matrix factorization~\citep{gunasekar2017_ImplicitRegularization,arora2019_ImplicitRegularization,li2021_ResolvingImplicit,razin2021_ImplicitRegularization},
linear networks~\citep{yun2021_UnifyingView,nacson2022_ImplicitBias}
and neural networks~\citep{kubo2019_ImplicitRegularization,woodworth2020_KernelRich}.
Moreover, variants of gradient descent such as stochastic gradient descent are also shown to have implicit regularization effects~\citep{li2022_WhatHappens,vivien2022_LabelNoise,pesme2021_ImplicitBias}.
While most of these works focus only on the optimization process and the final solution,
only a few works provided generalization guarantees for the over-parameterized models with early stopping.
Being the most relevant to our work,
they include over-parameterization in linear regression~\citep{zhao2022_HighdimensionalLinear,li2021_ImplicitSparse,vaskevicius2019_ImplicitRegularization} and single index model~\citep{fan2021_UnderstandingImplicit}.

%
%

\subsection{Notations}
We denote by $\ell^2 = \dk{(a_j)_{j\geq 1} \mid \sum_{j\geq 1}a_j^2 < \infty}$ the space of square summable sequences.
We write $a \lesssim b$ if there exists a constant $C>0$ such that $a \leq Cb$ and
$a \asymp b$ if $a \lesssim b$ and $b \lesssim a$,
where the dependence of the constant $C$ on other parameters is determined by the context.
We denote by $\abs{S}$ the cardinality of a set $S$.

\section{Limitations of Fixed Kernel Regression}\label{sec:kernel-regression}

Let us consider the non-parametric regression problem given by $y = f^*(x) + \ep$,
where $x \in \caX$ and $\caX$ is the input space with $\mu$ being a probability measure supported on $\caX$
and $\ep$ is a mean-zero $\sigma^2$-sub-Gaussian noise.
The function $f^*(x)$ represents the unknown regression function we aim to learn.
Suppose we are given samples $\{(x_i,y_i)\}_{i=1}^n$, drawn i.i.d.\ from the model.
We denote $X = (x_1,\dots,x_n)^\T$ and $Y = (y_1,\dots,y_n)^\T$.
For an estimator $\hat{f}$ of $f^*$, we consider the mean squared error defined as
$\E_{x \sim \mu} \zk{\hat{f}(x) - f^*(x)}^2 = \norm{\hat{f} - f^*}_{L^2(\caX,\dd \mu)}^2$.

Let $k: \caX \times \caX \to \R$ be a continuous positive definite kernel and $\caH_k$ be its associated reproducing kernel Hilbert space (RKHS).
The well-known Mercer's decomposition~\citep{steinwart2012_MercerTheorem} of the kernel function $k$ gives
\begin{align}
  \label{eq:Mercer}
  k(x,y) = \sum_{j=1}^\infty \lambda_j e_j(x) e_j(y),
\end{align}
where $(e_j)_{j \geq 1}$ is an orthonormal basis of $L^2(\caX,\dd \mu)$, and $(\lambda_j)_{j \geq 1}$ are the eigenvalues of $k$ in descending order.
Moreover, we can introduce the feature map (as a column vector)
\begin{align}
  \label{eq:FixedKernelFeatureMap}
  \Phi(x) = (\lambda_j^{\hf} e_j(x))_{j \geq 1} : \caX \to \ell^2
\end{align}
such that $k(x,x') = \ang{\Phi(x),\Phi(x')}$.
With the feature map, a function $f \in \caH_k$ can be represented as $f(x) = \ang{\Phi(x),\beta}_{\ell^2}$ for some $\beta \in \ell^2$.

Defining the empirical loss as
\begin{align}
  \label{eq:EmpiricalLoss}
  \caL_n(f) = \frac{1}{2n}\sum_{i=1}^n (y_i - f(x_i))^2,
\end{align}
we can consider an estimator $\hat{f}_t = \ang{\Phi(x),\beta_t}_{\ell^2}$ governed by the following gradient flow on the feature space
\begin{align}
  \label{eq:KernelGD_RKHS}
  \dot{\beta}_t = - \nabla_{\beta} \caL =  \frac{1}{n} \sum_{i=1}^n (y_i - \ang{\Phi(x_i),\beta_t}_{\ell^2}) \Phi(x_i),
  \qq{where} \beta_0 = \bm{0}.
\end{align}
This kernel gradient descent (flow) estimator corresponds to neural networks at infinite width limit by the celebrated neural tangent kernel (NTK) theory~\citep{jacot2018_NeuralTangent,allen-zhu2019_ConvergenceTheory}.

An extensive literature~\citep{yao2007_EarlyStopping,lin2018_OptimalRatesa,li2024_GeneralizationError}
has studied the generalization performance of such kernel gradient descent estimator.
From the Mercer's decomposition, we can further introduce interpolation spaces for $s \geq 0$ as
\begin{align}
  \label{eq:RKHS_Expansion}
  \zk{\caH_k}^s \coloneqq \Big\{ \sum_{j =1}^\infty \beta_j \lambda_j^{\frac{s}{2}} e_j \;\big|\; (\beta_j)_{j \geq 1} \in \ell^2 \Big\},
\end{align}
which is equipped with the norm $\norm{f}_{[\caH_k]^s} = \norm{\bm{\beta}}_{\ell^2}$ for $f = \sum_{j =1}^\infty \beta_j \lambda_j^{\frac{s}{2}} e_j$.
Particularly, the interpolation space $[\caH_k]^1$ corresponds to the RKHS $\caH_k$ itself.
Then, assuming the eigenvalue decay rate $\lambda_j \asymp j^{-\gamma}$,
the classical results~(see, e.g., \citet{yao2007_EarlyStopping,li2024_GeneralizationError}) in kernel regression state that
the optimal rate of convergence under the source condition $f^* \in [\caH_k]^s$ with $\norm{f^*}_{[\caH_k]^s} \leq 1$ is
$n^{-\frac{s\gamma}{s\gamma+1}}$.
However, since the interpolation space $[\caH_k]^s$ is defined via the eigen-decomposition of the kernel,
the generalization performance of kernel regression methods depends substantially on the eigen-decomposition of the kernel and the expansion of the target function under the basis.
Therefore, the choice of the kernel could significantly affect the performance of the kernel method.
To demonstrate this quantitatively, let us consider the following examples.

\begin{example}[Eigenfunctions in common order]
  \label{example:common-eigenfunctions}
  It is well known that kernels possessing certain symmetries, such as dot-product kernels on the sphere or translation-invariant periodic kernels on the torus,
  share the same set of eigenfunctions (such as the spherical harmonics or the Fourier basis).
  If we consider a fixed set of eigenfunctions $\{e_j\}_{j\geq 1}$ and a given truth function $f^*$,
  for two kernels $k_1$ and $k_2$ with eigenvalue decay rates $\lambda_{j,1} \asymp j^{-\gamma_1}$ and $\lambda_{j,2} \asymp j^{-\gamma_2}$ respectively,
  it follows that:
  \begin{align*}
    f^* \in [\caH_{k_1}]^{s_1} \Longleftrightarrow f^* \in [\caH_{k_2}]^{s_2} \qq{for} \gamma_1 s_1 = \gamma_2 s_2.
  \end{align*}
  Given that the convergence rate is dependent solely on the product $s \gamma$, the convergence rates relative to the two kernels will be identical.
\end{example}

\cref{example:common-eigenfunctions} seems to show that when the eigenfunctions are fixed,
kernel regression methods yield similar performance across different kernels.
However, it's important to note that this similarity is due to both kernels having \textit{the same eigenvalue decay order},
which aligns with the predetermined order of the basis.
In fact, if the eigenvalue decay order of a kernel deviates from that of the true function,
even if the eigenfunction basis remain the same, it can lead to significantly different convergence rates.
Let us consider the following example to illustrate this point.

\begin{example}[Low-dimensional structure]
  \label{example:low-dim-structure}
  Consider translation-invariant periodic kernels on the torus $\mathbb{T}^d = [-1,1)^d$,
  with eigenfunctions forming the Fourier basis $\phi_{\bm{m}}(\bm{x}) = \exp(i \pi \ang{\bm{m},\bm{x}})$,
  where $\bm{m} \in \mathbb{Z}^d$.
  Within this basis, a target function $f^*(\bm{x})$ can be represented as
  $f^* = \sum_{\bm{m} \in \mathbb{Z}^d} f_{\bm{m}} \phi_{\bm{m}}(\bm{x})$.
%
  Assuming $f^*$ exhibits a low-dimensional structure, specifically $f^*(\bm{x}) = g(x_1,\ldots,x_{d_0})$ for some $d_0 < d$,
  and considering $g$ belongs to the Sobolev space $H^t(\mathbb{T}^{d_0})$ of order $t$, the coefficients $f_{\bm{m}}$ can be shown to simplify to:
  \begin{align*}
    f_{\bm{m}} =
    \begin{cases}
      g_{\bm{m}_1}, & \bm{m} = (\bm{m}_1,\bm{0}),~ \bm{m}_1 \in \mathbb{Z}^{d_0}, \\
      0, & \text{otherwise}.
    \end{cases}
  \end{align*}
  Let us now consider two translation-invariant periodic kernels $k_1$ and $k_2$ given in terms of their eigenvalues:
  $k_1$ is given by $\lambda_{\bm{m},1} = (1 + \norm{\bm{m}}^2)^{-r}$ for some $r > d/2$,
  whose RKHS is the full-dimensional Sobolev space $H^r(\bbT^d)$;
  $k_2$ is given by $\lambda_{\bm{m},2} = (1 + \norm{\bm{m}}^2)^{-r}$ for $\bm{m} = (\bm{m}_1,\bm{0})$ and $\lambda_{\bm{m},2} = 0$ otherwise.
  Then, the function $f^*$ belongs to both $[\mathcal{H}_{k_1}]^{s}$ and $[\mathcal{H}_{k_2}]^{s}$ for $s= t / r$.
  After reordering the eigenvalues in descending order,
  the decay rates for the two kernels are identified as $\gamma_1 = 2r/d$ and $\gamma_2 = 2r/d_0$.
  Thus, the convergence rates with respect to the two kernels are respectively:
  \begin{align*}
    \frac{2t}{2t + d} \qand \frac{2t}{2t + d_0}.
  \end{align*}
  Therefore, we see that when $d$ is significantly larger than $d_0$, the convergence rate for the second kernel notably surpasses that of the first.


\end{example}

This example illustrates that the eigenvalues can significantly impact the learning rate, even when the eigenfunctions are the same.
In the scenario presented, the second kernel benefits from the low-dimensional structure of the target function by focusing only on the relevant dimensions,
whereas the first one suffers from the curse of dimensionality since it considers all dimensions.
The key point to take away from this example is the \textit{alignment between the kernel and the target function}.
To generalize this example, we can consider the following example where the order of the eigenvalues does not align with the order of the target function's coefficients.

\begin{example}[Misalignment]
  \label{example:misalignment}

  Let us fix a set of the eigenfunctions $\xk{e_j}_{j\geq 1}$
  and expand the truth function as $f^* = \sum_{j \geq 1} \theta_j^* e_j$.
  Note that by giving $\xk{e_j}_{j\geq 1}$, we already defined an order of the basis in $j$,
  but coefficients $\theta_j^*$ of the truth function are not necessarily ordered by $j$.
  Suppose that an index sequence $\ell(j)$ gives the descending order of $\abs{\theta_{\ell(j)}^*}$,
  namely $\abs{\theta_{\ell(1)}^*} \geq \abs{\theta_{\ell(2)}^*} \geq \dots$
  Then we can characterize the misalignment by the difference between $\ell(j)$ and $j$.
  Specifically, we assume that
  \begin{align}
    \label{eq:GappedDecay}
    \abs{\theta_{\ell(j)}^*} \asymp j^{-(p+1)/2} \qq{and} \ell(j) \asymp j^q \qq{for} p > 0,~q \geq 1,
  \end{align}
  where larger $q$ indicates a more severe misalignment.
  In terms of eigenvalues, let us consider $\lambda_{j,1} \asymp j^{-\gamma}$, which is in the order of $j$,
  while $\lambda_{\ell(j),2} \asymp j^{-\gamma}$, which is in the order of $\ell(j)$.
  Then, the convergence rates with the two sequences of coefficients are respectively
  \begin{align*}
    \frac{p}{p + q} \qand \frac{p}{p + 1}.
  \end{align*}
  Therefore, the convergence rates can differ greatly if the misalignment is significant, namely when $q$ is large.
\end{example}

From \cref{example:low-dim-structure} and \cref{example:misalignment},
we find that it is beneficial that \textit{the eigenvalues of the kernel align with the structure of the target function}.
However, one can hardly choose the proper kernel a priori, especially when the structure of the target function is unknown,
so the fixed kernel regression can be limited by the kernel itself and be unsatisfactory.
Motivated by these examples, we would like to explore the idea of an ``adaptive kernel approach,''
where the kernel can be learned from the data.

\section{Diagonal Adaptive Kernel Regression}\label{sec:diagonal-adaptive-kernel-regression}

Motivated by the examples in the last section,
as a first step toward the adaptive kernel approach, we consider \textit{adapting the eigenvalues of the kernel with eigenfunctions fixed}.
To this end, we first recall the fixed kernel feature map \cref{eq:FixedKernelFeatureMap} and the kernel gradient flow \cref{eq:KernelGD_RKHS}.
To learn the eigenvalues, let us consider the feature map parameterized by $\bm{a} = \xk{a_j}_{j \geq 1}$ and the corresponding kernel:
\begin{align}
  \label{eq:DiagonalFeatureMap}
  \Phi_{\bm{a}}(x) = \xk{a_j e_j(x)}_{j \geq 1},\quad
  k_{\bm{a}}(x,x') = \sum_{j \geq 1} a_j^2 e_j(x) e_j(x').
\end{align}
For the predictor, we consider similarly $f(x) = \ang{\Phi_{\bm{a}}(x), \bm{\beta}}_{\ell^2}$,
where $\bm{\beta} = \xk{\beta_j}_{j \geq 1} \in \ell^2$ is the coefficient vector under the feature map.
Then, we consider the gradient flow that trains both $\bm{a}$ and $\bm{\beta}$:
\begin{align}
  \label{eq:DiagonalKernelGD}
  \left\{
  \begin{aligned}
    \dot{\bm{\beta}}(t) &= - \nabla_{\bm\beta} \caL_n(f_t),\quad \bm{\beta}(0) = \bm{0}; \\
    \dot{\bm{a}}(t) &= - \nabla_{\bm a} \caL_n(f_t), \quad a_j(0) = \lambda_j^{1/2},
  \end{aligned}
  \right.
\end{align}
where $f_t = \ang{\Phi_{\bm{a}(t)}(x), \bm{\beta}(t)}_{\ell^2}$ and $\caL_n$ is the empirical loss defined in \cref{eq:EmpiricalLoss}.
Here, $(\lambda_j)_{j \geq 1}$ are the initial eigenvalues of the kernel.
We will denote by $\hat{f}^{\mr{EigGD}}_t = f_t$ the estimator obtained by the gradient flow \cref{eq:DiagonalKernelGD} at time $t$.

%

First, let us introduce the following assumption that the eigenfunctions are uniformly bounded and the eigenvalues decay polynomially.
This assumption is common in the kernel methods literature~\citep{steinwart2009_OptimalRates,zhang2023_OptimalityMisspecifieda},
and is satisfied, for example, the eigenfunctions are the Fourier basis and the kernel is associated with a Sobolev space~\citet{fischer2020_SobolevNorm}.

\begin{assumption}
  \label{assu:EigenSystem}
  We assume that $\sup_{j \geq 1} \norm{e_j(x)}_{\infty} \leq C_{\mr{eigf}}$
  and for some $\gamma > 1$,
  \begin{align}
    c_{\mr{eig}} j^{-\gamma} \leq \lambda_j \leq C_{\mr{eig}}j^{-\gamma},
  \end{align}
  where $C_{\mr{eigf}}, c_{\mr{eig}}, C_{\mr{eig}} > 0$ are constants.
\end{assumption}

For the truth function, since $\xk{e_j(x)}_{j \geq 1}$ forms an orthonormal basis,
we can write the truth function as
$f^*(x) = \sum_{j \geq 1} \theta_j^* e_j(x).$
Then, we introduce the following quantities on the truth coefficients $(\theta_j^*)_{j \geq 1}$.
\begin{align}
  \label{eq:PhiPsi}
  \begin{aligned}
    J_{\mr{sig}}(\delta) &\coloneqq \left\{ j : \abs{\theta_{j}^*} \geq \delta \right\}, \quad
    \Phi(\delta) \coloneqq \abs{J_{\mr{sig}}(\delta)},\\
    \Psi(\delta) &= \sum_{j \notin J_{\mr{sig}}(\delta)} (\theta_j^*)^2 = \sum_{j = 1}^\infty  (\theta_j^*)^2 \bm{1}\{(\theta_j^*)^2 < \delta^2\}.
  \end{aligned}
\end{align}
Our assumption on the truth coefficients is as follows.

\begin{assumption}
  \label{assu:SignalSpan}
  There exists constant $B_{\infty}$ such that $|\theta_j^*| \leq B_{\infty}$ for all $j \geq 1$.
  Moreover, there are constants $\kappa \geq 1$ and $B_{\mr{spn}} > 0$ such that
  \begin{align}
    \label{eq:SignificantSpan}
    \max J_{\mr{sig}}(\delta) \leq B_{\mr{spn}} \delta^{-\kappa}, \quad \forall \delta > 0.
  \end{align}
  and there is a constant $s_0, B_{\mr{sig}} > 0$ such that
  \begin{align}
    \label{eq:FastDecay}
    \Phi(\delta) \leq B_{\mr{sig}} \delta^{-(1-s_0)}.
  \end{align}
\end{assumption}

In \cref{assu:SignalSpan}, condition \cref{eq:SignificantSpan} says that the span of the significant components, namely those with $\abs{\theta_j^*} \geq \delta$,
grows at most polynomially in $1/\delta$.
This assumption is mild and holds for many practical settings, such as cases considered in \cref{example:misalignment} ($\kappa = \frac{2q}{p+1}$ for the first kernel).
In other perspective, it imposes a mild condition on the misalignment between the ordering of the truth signal and the ordering of the eigenvalues,
where $\kappa$ measures the misalignment between the ordering of $\theta_j$ and the ordering of $j$ itself.
On the other hand, the condition \cref{eq:FastDecay} requires that the significant components decay fast enough such that $(\theta_j^*)_{j \geq 1}$ is summable,
which is also a mild regularity condition.

With these assumptions, we can now present the generalization error of the diagonal adaptive kernel regression.
The proof of the following theorem is given in \cref{sec:proofs}.
We will discuss the implications of this result in \cref{subsec:discussions}.

\begin{theorem}
  \label{thm:EigenvalueGD}
  Suppose that \cref{assu:EigenSystem} and \cref{assu:SignalSpan} hold.
  Then, for any $s > 0$, by choosing $t \asymp n^{\hf}$,
  when $n$ is sufficiently large, with probability at least $1-C/n^2$,
  we have
  \begin{align}
    \norm{f^* - \hat{f}^{\mr{EigGD}}_t}_{L^2}^2
    \lesssim \frac{1}{n}\zk{ \Phi\xkm{n^{-1/2}\sqrt {\ln n}} \ln n + n^{\frac{1+s}{2\gamma}}} + \Psi\xkm{n^{-1/2} {\ln n}}.
  \end{align}
  Here, the constant $C$ and the hidden constants may depend on the constants in the assumptions and the choice of $s$.
\end{theorem}

\subsection{Deeper Over-parameterization}

In addition to the parameterized feature map in \cref{eq:DiagonalFeatureMap},
we can further consider adding extra parameters to the feature map to form a multilayer model.
Let us now consider the deeper parameterization
\begin{align}
  \Phi_{\bm{a},\bm{b}}(x) = \xk{a_j b_j^D e_j(x)}_{j \geq 1},\quad
  k_{\bm{a},\bm{b}}(x,x') = \sum_{j \geq 1} a_j^2 b_j^{2D} e_j(x) e_j(x'),
\end{align}
where $D > 0$ is the number of extra layers and $\bm{b} = \xk{b_j}_{j \geq 1}$ is the extra parameter.
We consider similarly $f(x) = \ang{\Phi_{\bm{a},\bm{b}}(x), \bm{\beta}}_{\ell^2}$
and the gradient flow on the parameters $\bm{a}$, $\bm{b}$, and $\bm{\beta}$:
\begin{align}
  \label{eq:DiagonalMultilayerGD}
  \left\{
  \begin{aligned}
    \dot{\bm{a}} &= -\nabla_{\bm{a}} \caL_n(f_t), \quad a_j(0) = \lambda_j^{\hf}; \\
    \dot{\bm{b}} &= -\nabla_{\bm{b}} \caL_n(f_t), \quad b_j(0) = b_0; \\
    \dot{\bm{\beta}} &= -\nabla_{\bm{\beta}} \caL_n(f_t), \quad \bm{\beta}(0) = \bm{0}.
  \end{aligned}
  \right.
\end{align}
Here, $b_0 > 0$ is the common initialization for all $b_j$'s.
We denote by $\hat{f}^{\mr{EigGD},D}_t$ the estimator obtained by this gradient flow \cref{eq:DiagonalMultilayerGD} at time $t$.

We remark here that if one considers the over-parameterization
$\Phi(x)=  \xk{a_j b_{j,1}\cdots b_{j,D} e_j(x)}_{j \geq 1}$ with the same initialization
$b_{j,k} = b_0$, $k = 1,\dots,D$,
then $b_{j,k}$'s remain to be the same during the training process by symmetry,
so this is equivalent to our parameterization $\theta_j = a_j b_j^D \beta_j$ modulo a constant factor.

Then, we have the following generalization error bound for the multilayer model.
The proof of the following theorem is deferred to \cref{sec:proofmulti}.

\begin{theorem}
  \label{thm:EigenvalueDeepGD}
  Suppose that \cref{assu:EigenSystem} and \cref{assu:SignalSpan} hold.
  Then, for any $s > 0$, by choosing $t \asymp n^{\frac{D+1}{D+2}}$ and $b_0 \asymp n^{-\frac{1}{2(D+2)}}$,
  when $n$ is sufficiently large, with probability at least $1-C/n^2$, we have
  \begin{align}
    \norm{f^* - \hat{f}^{\mr{EigGD},D}_t}_{L^2}^2
    \lesssim \frac{1}{n}\zk{ \Phi\xkm{n^{-1/2}\sqrt {\ln n}} \ln n + n^{\frac{1+s}{(D+2)\gamma}}} + \Psi\xkm{n^{-1/2} {\ln n}}
  \end{align}
  Here, the constant $C$ and the hidden constants may depend on the constants in the assumptions and the choice of $s$.
\end{theorem}

\subsection{Discussions} \label{subsec:discussions}

Let us discuss the implications of \cref{thm:EigenvalueGD} and \cref{thm:EigenvalueDeepGD} in the following aspects.

\subsubsection{Benefits of Over-parameterization}
\cref{thm:EigenvalueGD} and \cref{thm:EigenvalueDeepGD} demonstrate the advantage of over-parameterization in the sequence model.
Compared with the vanilla fixed-eigenvalues gradient descent method,
the over-parameterized gradient descent method can significantly improve the generalization performance by adapting the eigenvalues to the truth signal.
For a more concrete example,
if we consider the setting of \cref{eq:GappedDecay}, we have
$\Phi(\delta) \asymp \delta^{-\frac{2}{p+1}}, \Psi(\delta) \asymp \delta^{\frac{2p}{p+1}}$ (for a proof, see Lemma F.2 in \citet{li2024_ImprovingAdaptivity})
and the following corollary.

\begin{corollary}
  \label{cor:OpVsGD}
  Consider eigenvalue adaptive kernel regression in \cref{eq:DiagonalKernelGD} or \cref{eq:DiagonalMultilayerGD}.
  Suppose \cref{assu:EigenSystem} and \cref{eq:GappedDecay} hold and $\gamma > \frac{p+1}{D+2}$.
  Then, by choosing $b_0 \asymp n^{-\frac{1}{2(D+2)}}$ (if $D\neq 0$) and $t \asymp n^{\frac{D+1}{D+2}}$,
  when $n$ is sufficiently large, omitting the logarithmic factor, we have
  \begin{align}
    \norm{f^* - \hat{f}^{\mr{EigGD},D}_t}_{L^2}^2
    =\tilde{O}_{\bbP}(n^{-\frac{p}{p+1}} ).
  \end{align}
  In comparison, the vanilla kernel gradient flow method yields the convergence rate $\Theta(n^{-\frac{p}{p+q}})$.
\end{corollary}

\cref{cor:OpVsGD} shows that the over-parameterized gradient descent method can achieve the optimal rate $n^{-\frac{p}{p+1}}$,
while the vanilla gradient descent method only achieves the rate $n^{-\frac{p}{p+q}}$.
The improvement is significant when $q$ is large, which corresponds to the case that the misalignment between the ordering of the truth signal and the ordering of the eigenvalues is severe.
Moreover, if we return to the low-dimensional regression function in \cref{example:low-dim-structure} with the isotropic kernel $k_1$,
we can see that while the vanilla gradient descent method suffers from the curse of dimensionality with the rate $\frac{2t}{2t+d}$,
the over-parameterization leads to the dimension-free rate $\frac{2t}{2t+d_0}$.
Therefore, the over-parameterization significantly improves the generalization performance.

\subsubsection{Learning the Eigenvalues}
To further investigate how the eigenvalues are adapted by over-parameterized gradient descent, we present the following proposition.

\begin{proposition}
  \label{prop:EigLearn}
  Given the same conditions as in \cref{thm:EigenvalueGD} or \cref{thm:EigenvalueDeepGD}, for the chosen stopping time $t$,
  the term learning the eigenvalues $a_k(t) b_l^D(t)$ (taking $D=0$ and $b_j^D = 1$ for \cref{thm:EigenvalueGD}) satisfies:
  \begin{itemize}
    \item {Signal component:} For $\theta_k^*$ such that $|\theta_k^*| \geq C_1 n^{-1/2}  \ln n$,
    we have
    \begin{align}
      a_k(t) b^D_k(t) \geq c_1 \abs{\theta_k^*}^{\frac{D+1}{D+2}}.
    \end{align}
    \item {Noise component:} For components where $\abs{\theta_k^*} \leq n^{-1/2}$ and $\lambda_k \leq n^{-\frac{1+s}{D+2}}$, we have
    \begin{align}
      a_k(t) b^D_k(t) \leq C_2 a_k(0) b_k^D(0) \exp(C_3(\sqrt{\ln n} + \sqrt{\ln k})).
    \end{align}
  \end{itemize}
  Here, $c_1,C_1,C_2,C_3$ are all constants not depending on either $n$ or $k$.
\end{proposition}

From this proposition, we can see that for the signal components, the eigenvalues are learned to be at least a constant times a certain power of the truth signal magnitude.
Thus, over-parameterized gradient descent adjusts the eigenvalues to match the truth signal.
In the case of noise components, although the eigenvalues are still increasing due to the training process,
the eigenvalues do not exceed the initial values by some constant factor, provided that $\lambda_j$ is relatively small.
This finding suggests that over-parameterized gradient descent effectively adapts eigenvalues to the truth signal while mitigating overfitting to noise,
showing a feature learning behavior.
We remark that when $\lambda_j$ is relatively large, the method still tends to overfit the noise components,
contributing an extra $n^{\frac{1}{D+2} \frac{1}{\gamma}}$ term in the generalization error,
but this term becomes negligible when $\gamma$ is large.

\subsubsection{Adaptive Choice of the Stopping Time}
A notable advantage of the adaptive kernel approach is that it does not require the knowledge of the truth parameters to choose the stopping time.
Consider the scenario described by \cref{eq:GappedDecay},
vanilla gradient descent requires the selection of a stopping time $t \asymp n^{(q\gamma)/(p+q)}$ to achieve the optimal convergence rate.
However, this choice of stopping time critically depends on the unknown parameters $p$ and $q$ of the truth parameter.
In contrast, the over-parameterized gradient descent only need to choose the stopping time as $t \asymp n^{\frac{D+1}{D+2}}$,
which does not rely on the unknown truth parameters.

\subsubsection{Effect of the Depth}
The results in \cref{thm:EigenvalueDeepGD} also show that deeper over-parameterization can further improve the generalization performance.
In the two-layer over-parameterization,
the extra term $n^{\frac{1+s}{2\gamma}}$ in \cref{thm:EigenvalueGD} emerges due to the limitation of the adapting large eigenvalues.
With the introduction of depth, namely adding extra $D$ layers to the model with proper initialization,
this term can be improved to $n^{\frac{1+s}{(D+2)\gamma}}$ in \cref{thm:EigenvalueDeepGD}.
This improvement suggests that the depth can refine the model's sensitivity to eigenvalue adaptation,
enabling a more nuanced adjustment to the underlying signal structure.
This finding underscores the importance of model depth in enhancing the learning process,
providing also theoretical evidence for the empirical success of deep learning models.

\subsubsection{Comparison with Previous Works}

Let us compare our results with the existing literature~\citep{zhao2022_HighdimensionalLinear,li2021_ImplicitSparse,vaskevicius2019_ImplicitRegularization}
on the generalization performance of diagonal over-parameterized gradient descent in the following aspects:
\begin{itemize}
 \item \textbf{Problem settings:} While the existing literature~\citep{zhao2022_HighdimensionalLinear,li2021_ImplicitSparse,vaskevicius2019_ImplicitRegularization}
 investigate the realms of high-dimensional linear regression, focusing on implicit regularization and sparsity,
 the present study dives into non-parametric kernel regression,
 emphasizing the adaptivity of trainable kernel to the underlying signal's structure.

 \item \textbf{Over-parameterization setup:}
 The existing work \citet{zhao2022_HighdimensionalLinear} considers the over-parameterization setup by the two-layer Hadamard product $\theta = a \odot b$
 where the initialization is the same for each component that $a(0) = \alpha \bm{1}$ and $b(0) = \bm{0}$.
 In comparison, our work considers initializing the eigenvalues $a_j(0) = \lambda_j^{1/2}$ differently for each component.
 Moreover, we extend the over-parameterization to deeper models by adding extra $D$ layers.
 Although the subsequent work \citet{li2021_ImplicitSparse} of \citet{vaskevicius2019_ImplicitRegularization} considered deeper over-parameterization,
 their over-parameterization is in the form of $\theta = u^{\odot D} - v^{\odot D}$ with $u(0) = v(0) = \alpha \bm{1}$.
 Though being easy to analysis because of the homogeneous initialization,
 this setup could not bring insights into the learning of the eigenvalues, which is a key insight in our results.
 Furthermore, the analysis for \cref{thm:EigenvalueDeepGD} involves the interplay between the differently initialized $a_j$ and $b_j$,
 so our analysis is more involved than the existing works.
 We also remark that although we only consider the gradient flow in the analysis,
 the results can be extended to the gradient descent with proper learning rates.

  \item \textbf{Interpretation of the over-parameterization:}
 The previous works view the over-parameterization mainly as a mechanism for implicit regularization,
 while our work provides a novel perspective that over-parameterization adapts to the structure of the truth signal by learning the eigenvalues.
 As shown in \cref{prop:EigLearn}, the parameterization of the kernel demonstrates a feature learning behavior,
 where the eigenvalues are learned to match the truth signal.

 \item \textbf{The restricted isometry property (RIP) condition:}
 The existing literature relies on the RIP condition to analyze the gradient dynamics.
 In comparison, our work does not require the RIP condition by using concentration of the eigenfunctions and a delicate analysis of the gradient flow dynamics.
 Moreover, we remark that as discussed in \citet{vaskevicius2019_ImplicitRegularization},
 the RIP condition in \citet{zhao2022_HighdimensionalLinear,li2021_ImplicitSparse} is in fact too strong to hold for i.i.d.\ random samples.

\end{itemize}


\section{Proofs}\label{sec:proofs}

In this section, we provide the proof of \cref{thm:EigenvalueGD},
while the proof of \cref{thm:EigenvalueDeepGD}, being similar but more complicated, is deferred to \cref{sec:proofmulti}

The main idea of the proof is to approximate the dynamics \cref{eq:TwoLayer_GF_Matrix}, the explicit form of \cref{eq:DiagonalKernelGD}, by a diagonalized version,
so we can analyze the dynamics of each component separately.
Then, we can use the idea in the sequence model~\citep{li2024_ImprovingAdaptivity} (the conference paper) to deal with signal components and noise components separately,
showing that the signal components converge to the truth parameter at the optimal rate and the noise components are well controlled.
The main challenge here is to deal with an extra perturbation term:
we analyze the dynamics of the diagonalized equation with perturbation and use an iterative argument with \cref{prop:TwoLayer_Shrinkage}
to gradually shrink the perturbation term.

The rest of this section is organized as follows:
In \cref{subsec:Proof_ExplicitForm}, we derive the explicit form \cref{eq:TwoLayer_GF_Matrix}  of the gradient flow.
In \cref{subsec:Proof_ComponentPartition} and \cref{subsec:Proof_Concentration},
we establish the approximation between the original dynamics and the diagonalized one.
In \cref{subsec:Proof_TwoLayer_OneDim}, we analyze the dynamics of the one-dimensional equation with perturbation.
Finally, we prove \cref{thm:EigenvalueGD} in \cref{subsec:Proof_MainThm}.

In the proof, we use $s$ to represent a small positive constant that can be chosen arbitrarily small.
We use $C,c$ to represent positive constants that can depend on the constants in the assumptions and also the choice of $s$.
The values of $C,c,s$ may change from line to line.
We denote $\ln^+(x) = \max(\ln x, 0)$.
Let $S \subseteq \bbN^+$ be an index set.
For a vector $\bm{v}$, we denote by $\bm{v}_S$ the sub-vector of $v$ with indices in $S$.
Let $A$ be a matrix and $R \subseteq \bbN^+$ be another index set.
We use similar notation $A_{SR}$, $A_{\cdot R}$ and $A_{kR}$ for sub-matrices.
Moreover,
we denote by $D_{\bm{v}}$ the diagonal matrix with diagonal entries $v$ and $\bm{a} \odot \bm{b}$ the element-wise product of two vectors $\bm{a}, \bm{b}$.

\subsection{Explicit Form of the Gradient Flow}\label{subsec:Proof_ExplicitForm}
Let $(e_j(x))_{j \geq 1}$ be the fixed orthonormal basis.
We denote (as a column vector) $E(x) = \xk{e_j(x)}_{j \geq 1} \in \ell^2.$
Also, for the adaptive kernel regression in \cref{eq:DiagonalKernelGD}, we denote
$\bm{\theta}(t) = \xk{\theta_j}_{j \geq 1} = \bm{a}(t) \odot \bm{\beta}(t).$
Then,
\begin{align*}
  f_t(x) = \ang{\Phi_{\bm{a}(t)}(x), \bm{\beta}(t)} = \ang{E(x), \bm{a}(t) \odot \bm{\beta}(t)} = \ang{E(x), \bm{\theta}(t)} = E(x)^{\T} \bm{\theta}(t),
\end{align*}
In comparison, the truth function is given similarly by $f^*(x) = \sum_{j \geq 1} \theta_j^* e_j(x) = E(x)^{\T} \bm{\theta}^*.$

Furthermore, we denote
\begin{align}
  \hat{\Sigma}  &= \frac{1}{n}\sum_{i=1}^n E(x_i) E(x_i)^{\T}, \quad
  {\bm{r}} = \frac{1}{n}\sum_{i=1}^n \ep_i E(x_i).
\end{align}
A direct computation shows that
\begin{align*}
  \nabla_{\beta}\caL_n(f_t) &= \frac{1}{n}\sum_{i=1}^n D_{\bm{\beta}(t)} E(x_i) \zk{E(x_i)^{\T} \xk{\bm{\theta}(t) - \bm{\theta}^*} - \ep_i}\\
  &= D_{\bm{\beta}(t)} \zk{\frac{1}{n}\sum_{i=1}^n E(x_i) E(x_i)^{\T} \xk{\bm{\theta}(t) - \bm{\theta}^*} - \frac{1}{n}\sum_{i=1}^n \ep_i E(x_i)} \\
  &= D_{\bm{\beta}(t)} \zk{{\hat{\Sigma}} \xk{\bm{\theta}(t) - \bm{\theta}^*} - {\bm{r}}}.
\end{align*}
Similarly,
\begin{align*}
  \nabla_{\bm{a}}\caL_n(f_t) = D_{\bm{a}(t)} \zk{{\hat{\Sigma}} \xk{\bm{\theta}(t) - \bm{\theta}^*} - {\bm{r}}}.
\end{align*}
Therefore, we derive the explicit form of the gradient flow as
\begin{align}
  \label{eq:TwoLayer_GF_Matrix}
  \left\{
  \begin{aligned}
    \dot{\bm{\beta}}(t) &= - \nabla_{\bm{\beta}} \caL_n = \bm{a}\odot \zk{\hat{\Sigma}(\bm{\theta}^* - \bm{\theta}) + {\bm{r}}}, \quad
    \bm{\beta}(0) = \bm{0},\\
    \dot{\bm{a}}(t) &= - \nabla_{\bm{a}} \caL_n = \bm{\beta} \odot \zk{\hat{\Sigma}(\bm{\theta}^* - \bm{\theta}) + {\bm{r}}},
    \quad a_j(0) = \lambda_j^{1/2}.
  \end{aligned}
  \right.
\end{align}

\subsection{Component Partition}\label{subsec:Proof_ComponentPartition}
Let us partition the components of $\bm{\theta}$ into two sets: the signal components $S \subset \bbN^+$ and the noise components $R = S^{\complement}$.
We choose $S$ by
\begin{align}
\label{eq:Proof_ComponentPartition}
  S = S_1 \cup S_2 = \dk{j \geq 1 : \abs{\theta^*_j} \geq n^{-1/2}\sqrt {\ln n}} \cup
  \dk{j\geq 1: \lambda_j \geq n^{-1/2}}.
\end{align}
Then, by \cref{assu:SignalSpan}, we have
\begin{align*}
  \abs{S_1}  \leq \Phi(n^{-1/2}) \leq C n^{1/2-s} \qquad
  \max S_1  \leq C n^{\kappa/2}.
\end{align*}
Also, by \cref{assu:EigenSystem}, we have $\abs{S_2} \leq C n^{1/(2\gamma)}.$
Noticing that $\gamma > 1$, we get
\begin{align}
  \label{eq:Proof_Component_S}
  \abs{S} \leq C n^{(1-s_0)/2} \qquad
  \max S  \leq C n^{\kappa/2}.
\end{align}
Moreover, \cref{assu:EigenSystem} also yields
$S_2^{\complement} \geq c n^{1/(2\gamma)}$, so
\begin{align}
  \label{eq:Proof_Component_L}
  L \coloneqq \min R \geq c n^{1/(2\gamma)}.
\end{align}

\noindent\textit{Approximation decomposition.}

We can approximate $\hat{\Sigma}(\bm{\theta}^* - \bm{\theta}) + {\bm{r}}$ by the diagonalized one $\bm{\theta}^* - \bm{\theta}$ by the following decomposition:
\begin{align}
  \label{eq:ToDiagonalDecomposition}
  \hat{\Sigma}(\bm{\theta}^* - \bm{\theta}) + \bm{r} = (\bm{\theta}^* - \bm{\theta}) + \bm{p} + \bm{q} + \bm{r}
   = (\bm{\theta}^* - \bm{\theta}) + \bm{h}
\end{align}
where $\bm{h} = \bm{p} + \bm{q} + \bm{r}$,
\begin{align}
  \bm{p} = (\hat{\Sigma}_{\cdot S} - I_{\cdot S})(\bm\theta^* - \bm\theta)_S,\quad
  \bm{q} = (\hat{\Sigma}_{\cdot R} - I_{\cdot R}) (\theta^* - \theta)_R,
\end{align}
and $I$ is the infinite-dimensional identity matrix.
In detail, we obtain \cref{eq:ToDiagonalDecomposition} by taking the difference of the following:
\begin{align*}
  \hat{\Sigma}(\bm{\theta}^* - \bm{\theta})
  &= \hat{\Sigma}_{\cdot S}(\bm{\theta}^* - \bm{\theta})_S + \hat{\Sigma}_{\cdot R}(\bm{\theta}^* - \bm{\theta})_R, \\
  \bm{\theta}^* - \bm{\theta}
  &= I_{\cdot S}(\bm{\theta}^* - \bm{\theta})_S + I_{\cdot R}(\bm{\theta}^* - \bm{\theta})_R.
\end{align*}

\subsection{Concentrations}\label{subsec:Proof_Concentration}

The following lemma give preliminary bounds for the error terms in \cref{eq:ToDiagonalDecomposition}.

\begin{lemma}
  \label{prop:ConcentrationQuantities}
  Under assumption \cref{assu:EigenSystem}, there are constants $C_1,C_2 > 0$ such that with probability at least $1-C_1/n^2$,
  for all $k \geq 1$, we have
  \begin{align}
    \label{eq:Concen_rk}
    & \abs{r_k} \leq C_2\sqrt {\frac{\ln (nk)}{n}} \\
    \label{eq:Concen_Sigma_kR_Theta}
    & \abs{(\hat{\Sigma}_{kR} - I_{kR}) \bm{\theta}^*_R} \leq C_2\sqrt {\frac{\norm{\bm{\theta}^*_R}_1}{n} \ln (nk)}, \\
    \label{eq:Concen_Sigma_kS}
    & \norm{\hat{\Sigma}_{kS} - I_{kS}}_2 \leq C_2 \sqrt {\frac{\abs{S}}{n} \ln (nk)} \\
    \label{eq:Concen_Sigma_kR_Lambda}
    & \norm{(\hat{\Sigma}_{kR} - I_{kR})\Lambda_R^{\hf}}_2 \leq C_2 \sqrt {\frac{\Tr \Lambda_R}{n} \ln (nk)}.
  \end{align}
\end{lemma}
\begin{proof}
  We first establish the inequalities for fixed $k$ and finally use the union bound to get the result for all $k$.
  For the first inequality, we recall that
  \begin{align*}
    r_k = \frac{1}{n}\sum_{i=1}^n \ep_i e_k(x_i).
  \end{align*}
  Since $(x_i,\ep_i)$s are i.i.d., $\abs{e_k(x_i)} \leq C_{\mr{eigf}}$ and $\ep_i$ is sub-Gaussian, we can apply Hoeffding's inequality \cref{lem:HoeffdingSubG} to get
  for fixed $k \geq 1$,
  \begin{align*}
    \bbP \dk{\abs{r_k} \geq C_2 \sqrt {\frac{\ln (nk)}{n}}} \leq \frac{C}{k^2 n^2}.
  \end{align*}

  For the second one, we note that
  \begin{align*}
    \hat{\Sigma}_{kR} \bm{\theta}^*_R
    = \frac{1}{n}\sum_{i=1}^n e_k(x_i)\sum_{j \in R} e_j(x_i) \theta^*_j.
  \end{align*}
  Denoting $\xi_i = e_k(x_i)\sum_{j \in R} e_j(x_i) \theta^*_j$, we have
  \begin{align*}
    \abs{\xi_i} \leq C_{\mr{eigf}}^2 \sum_{j \in R} \abs{\theta^*_j} = \norm{\bm{\theta}^*_R}_1,
  \end{align*}
  and $\E \xi_i = I_{kR} \bm{\theta}^*_R$.
  Therefore, the Hoeffding's inequality yields
  \begin{align*}
    \bbP \dk{\abs{(\hat{\Sigma}_{kR} - I_{kR}) \bm{\theta}^*_R} \geq C_2 \sqrt {\frac{\norm{\bm{\theta}^*_R}_1}{n} \ln (nk)}} \leq \frac{C}{k^2 n^2}.
  \end{align*}

  For the third inequality, we note that
  \begin{align*}
    \hat{\Sigma}_{kS} = \frac{1}{n}\sum_{i=1}^n \xk{e_k(x_i) e_j(x_i)}_{j \in S}.
  \end{align*}
  Defining the random vector $v = \xk{e_k(x_i) e_j(x_i)}_{j \in S}$, we have
  \begin{align*}
    \norm{v}_2^2 = \sum_{j \in S} \abs{e_k(x_i) e_j(x_i)}^2 \leq C_{\mr{eigf}}^2 \abs{S}.
  \end{align*}
  Consequently, using the Hoeffding's inequality on the Hilbert space, \cref{lem:HoeffdingHilbert}, we get
  \begin{align*}
    \bbP \dk{\norm{\hat{\Sigma}_{kS} - I_{kS}}_2 \geq 2 \sqrt {\frac{\abs{S}}{n} \ln (nk)}} \leq \frac{2}{n^2 k^2}.
  \end{align*}

  For the fourth one, similarly, we have
  \begin{align*}
    \hat{\Sigma}_{kR} \Lambda^{\hf}_R = \frac{1}{n}\sum_{i=1}^n \xk{e_k(x_i)  \lambda_j^{\hf} e_j(x_i)}_{j \in R}.
  \end{align*}
  Defining the random vector $w = \xk{e_k(x_i)  \lambda_j^{\hf} e_j(x_i)}_{j \in R}$, we have
  \begin{align*}
    \norm{w}_2^2 = \sum_{j \in R} \abs{e_k(x_i)  \lambda_j^{\hf} e_j(x_i)}^2 \leq C_{\mr{eigf}}^2 \sum_{j \in R} \lambda_j = \Tr \Lambda_R,
  \end{align*}
  so we also have
  \begin{align*}
    \bbP \dk{\norm{(\hat{\Sigma}_{kR} - I_{kR})\Lambda_R^{\hf}}_2 \geq 2 \sqrt {\frac{\Tr \Lambda_R}{n} \ln (nk)}} \leq \frac{2}{n^2 k^2}.
  \end{align*}

  Finally, noticing that the tail probability is summable over $k$, we can apply the union bound to get the result.

\end{proof}

\subsection{Proof of \cref{thm:EigenvalueGD}}\label{subsec:Proof_MainThm}


For some $E > 0$ that will be determined later, we define the time before overfitting the noise as
\begin{align}
  \label{eq:ErrorTime}
  T^{\mr{err}} = \inf \dk{t \geq 0 : \abs{\theta_k(t)} \leq 2 \lambda_k \exp(E \sqrt {\ln n + \ln k}),\quad \forall k \in R},
\end{align}
We claim that we always have $T^{\mr{err}} \geq \bar{C}\sqrt {n}$ and thus $t \leq T^{\mr{err}}$ for all the following times considered;
we will prove this claim at the end of this subsection, \cref{subsubsec:Proof_Noise}.
In the following, we will always assume that $t \leq T^{\mr{err}}$ even without explicitly mentioning it.

Now, we start with controlling the error terms $\bm{p}, \bm{q}, \bm{r}$ in \cref{eq:ToDiagonalDecomposition} using the concentration results in \cref{prop:ConcentrationQuantities}.
For $p_{k}$, we have
\begin{align*}
  \abs{p_k} &\leq \norm{\hat{\Sigma}_{kS} - I_{kS}}_{2} \norm{(\bm\theta^* - \bm\theta)_S}_{2}
  \stackrel{\cref{eq:Concen_Sigma_kS}}{\leq} C \sqrt {\frac{\abs{S}}{n} \ln (kn)} \cdot \sqrt {\abs{S}}\norm{(\bm\theta^* - \bm\theta)_S}_{\infty} \\
  &= C \abs{S} \sqrt{\frac{\ln (nk)}{n}} \cdot \norm{(\bm\theta^* - \bm\theta)_S}_{\infty},
\end{align*}
On the other hand, for $q_k$, we have
\begin{align*}
  q_k = (\hat{\Sigma}_{kR}- I_{\cdot R})\bm{\theta}^*_R - (\hat{\Sigma}_{kR}- I_{\cdot R})\bm{\theta}_R
\end{align*}
For the first term, we use \cref{eq:Concen_Sigma_kR_Theta} and also \cref{assu:SignalSpan} to get
\begin{align*}
  \abs{(\hat{\Sigma}_{kR} - I_{kR}) \bm{\theta}^*_R} \leq C \sqrt {\frac{\norm{\bm{\theta}^*_R}_1}{n} \ln (nk)} \leq C \sqrt{\frac{\ln (nk)}{n}}.
\end{align*}
For the second term, we use \cref{eq:Concen_Sigma_kR_Lambda}:
\begin{align*}
  \abs{(\hat{\Sigma}_{kR} - I_{kR}) \bm{\theta}_R} &= \abs{(\hat{\Sigma}_{kR}- I_{k R})\Lambda_R^{\hf} \Lambda_R^{-\hf}\bm{\theta}_R}
  \leq \norm{(\hat{\Sigma}_{kR}- I_{k R})\Lambda_R^{\hf}}_2 \norm{\Lambda_R^{-\hf}\bm{\theta}_R}_2\\
  &\leq C \sqrt{\frac{\Tr\Lambda_R}{n} \ln (nk)} \cdot \norm{\Lambda_R^{-\hf}\bm{\theta}_R}_2,
\end{align*}
Now, when $t \leq T^{\mr{err}}$, we have
\begin{align*}
  \norm{\Lambda_R^{-\hf}\bm{\theta}_R(t)}_2^2 &\leq C \sum_{k \in R} \lambda_k \exp(2E \sqrt {\ln n + \ln k}) \\
  &\leq C \exp(2E \sqrt {\ln n}) \sum_{k \geq L}\lambda_k \exp(2E \sqrt {\ln k}) \\
  & \leq C n^{s} \sum_{k \geq L} k^{-\gamma+s}  \leq C_E n^{s} L^{-(\gamma-1-s)} \\
  & \leq C_E n^{s} n^{-(\gamma-1-s)/(2\gamma)} = C_E n^{-(\gamma-1)/(2\gamma) + s} \leq C,
\end{align*}
where $C_E$ is a constant with extra dependence on $E$ and we can let $n$ be large enough so the resulting constant $C$ do not depend on $E$.
Hence, we have
\begin{align*}
  \abs{(\hat{\Sigma}_{kR} - I_{kR}) \bm{\theta}_R} \leq C \sqrt{\frac{\Tr\Lambda_R}{n} \ln (nk)} \leq C \sqrt{\frac{\ln (nk)}{n}}.
\end{align*}

Combining with \cref{eq:Concen_rk},
we summarize the bounds for $p_k, q_k, r_k$ as
\begin{align}
  \label{eq:Proof_BasicControls}
  \begin{aligned}
    \abs{p_k} \leq \eta_k \norm{(\bm\theta^* - \bm\theta)_S}_{\infty}, \quad
    \abs{q_k}  \leq C n^{-1/2}\sqrt{\ln (nk)}, \quad
    \abs{r_k} \leq C n^{-1/2}\sqrt{\ln (nk)},
  \end{aligned}
\end{align}
where we use \cref{eq:Proof_Component_S} to get
\begin{align*}
  \eta_k = C \abs{S} n^{-1/2}\sqrt{\ln (nk)} \leq C n^{-s} \sqrt{\ln (nk)}.
\end{align*}

\subsubsection{Shrinkage Dynamics for Signal Components}
\label{subsubsec:Proof_Shrinkage}
For the signals $k \in S$, we use $k \leq \max S \leq C n^{\kappa/2}$ to rewrite \cref{eq:Proof_BasicControls} as
\begin{align}
  \label{eq:Proof_BasicControls_S}
  \norm{\bm{p}_S}_{\infty} \leq \eta \norm{(\bm\theta^* - \bm\theta)_S}_{\infty},\qquad
  \norm{\bm{q}_S + \bm{r}_S}_{\infty} \leq \ep \coloneqq C n^{-1/2}\sqrt{\ln n}.
\end{align}
where we can choose $\eta = C n^{-s_1}$ for some small $s_1 >0$.
However, the challenge here is that the rough bound on $\norm{\bm{p}_S}_{\infty}$ via $\norm{(\bm\theta^* - \bm\theta)_S}_{\infty} \leq C$
does not give the desired optimal error bound $n^{-1/2}\sqrt{\ln n}$.
Nevertheless, if $\norm{(\bm\theta^* - \bm\theta)_S}_{\infty}$ decreases, $\norm{\bm{p}_S}_{\infty} $ also decreases.
Therefore, we need the following proposition to ensure the decrease of $\norm{(\bm\theta^* - \bm\theta)_S}_{\infty}$,
which shares similar ideas with the proof in \citet[Section B.3]{vaskevicius2019_ImplicitRegularization}.

\begin{proposition}[Shrinkage monotonicity and shrinkage time]
  \label{prop:TwoLayer_Shrinkage}
  Suppose that
  \begin{align*}
    \norm{\bm{p}_S}_{\infty} \leq \eta \norm{\bm\theta^*_S - \bm\theta_S}_{\infty},
    \quad
    \norm{\bm{q}_S + \bm{r}_S}_{\infty} \leq \ep,
  \end{align*}
  for some constant $\eta \leq 1/8$ and $\ep > 0$.
  Suppose for some $t_0 \geq 0$ that
  \begin{align*}
    \norm{\bm\theta^*_S - \bm\theta_S(t_0)}_{\infty} \leq M.
  \end{align*}
  Then,
  \begin{align}
    \label{eq:EqTwoLayer_ShrinkageBound}
    \norm{\bm\theta^*_S - \bm\theta_S(t)}_{\infty} \leq M \vee 2 \ep,\quad \forall t \geq t_0.
  \end{align}
  Furthermore,
  \begin{itemize}
    \item If $\eta M \geq \ep$, then
    \begin{align}
      \norm{\bm\theta^*_S - \bm\theta_S(t)}_{\infty} \leq \frac{1}{2} M,\quad \forall t \geq t_0 + \overline{T}^{\mr{half}}(M),
    \end{align}
    where
    \begin{align}
      \overline{T}^{\mr{half}}(M) = 4 M^{-1} \zk{2 + \frac{1}{2} \max_{k \in S} \xk{\ln^+ \frac{M}{\lambda_k}+ \ln^+\frac{\abs{\theta^*_k}}{\lambda_k}}}.
    \end{align}

    \item If $\eta M < \ep$, then for each $k \in S$ such that $\abs{\theta^*_k} \geq 4\ep$, we have
    \begin{align}
      \abs{\theta^*_k - \theta_k(t)} \leq 4 \ep,\quad \forall t \geq t_0 + \overline{T}^{\mr{fin}}_k,
    \end{align}
    where
    \begin{align}
      \overline{T}^{\mr{fin}}_k &\leq 4\abs{\theta^*_k}^{-1} \xk{2 + \frac{1}{2}\ln^+\frac{M}{\lambda_k} + \frac{1}{2}\ln^+ \frac{\abs{\theta^*_k}}{4\lambda_k} + \ln^+ \frac{\abs{\theta^*_k}}{2\ep}}.
    \end{align}
    On the other hand, for all $k \in S$ we have
    \begin{align}
      \label{eq:EqTwoLayer_ShrinkageBound_FinalSmallSig}
      \abs{\theta^*_k - \theta_k(t)} \leq 2\max(\abs{\theta^*_k},4\ep),\quad \forall t \geq t_0 + (2\ep)^{-1}.
    \end{align}
  \end{itemize}

\end{proposition}
\begin{proof}
  The proof depends on the analysis of the one-dimensional equation established in \cref{subsec:Proof_TwoLayer_OneDim}.
  We define
  \begin{align*}
    t = \inf\dk{s \geq t_0 : \norm{\bm\theta^*_S - \bm\theta_S(s)}_{\infty} > M \vee 2\ep}.
  \end{align*}
  Then, we must have $\norm{\bm\theta^*_S - \bm\theta_S(t)}_{\infty} = M \vee 2\ep$,
  which implies that $\abs{\theta_k^* - \theta_k(t)} = M \vee 2\epsilon$ for some index $k$.
  The bound then gives
  \begin{align*}
    \norm{\bm{h}_S(t)}_{\infty} \leq \kappa\coloneqq \eta (M \vee 2\ep) + \ep < M \vee 2\ep
  \end{align*}
  where we note that $\eta < 1/2$.
  Similar to the proof of \cref{lem:TwoLayer_Monotonicity}, we recall \cref{eq:TwoLayerEq_Theta} that
  \begin{align*}
    \dot{\theta}_k(t) = (a_k^{2}(t) + \beta_k^{2}(t)) (z - \theta_k(t) + h_k(t)).
  \end{align*}
  So, $\dot{\theta}_k(t)$ has the same sign as $z - \theta_k(t)$ and thus $\abs{\theta_k^* - \theta_k(t)}$ must be decreasing at $t$,
  which contradicts the definition of $t$.

  Now, we prove the second part.
  Using the bound \cref{eq:EqTwoLayer_ShrinkageBound} that we have obtained, we get
  \begin{align*}
    \norm{\bm{h}_S(t)}_{\infty} \leq \kappa,\quad \forall t \geq t_0.
  \end{align*}
  If $\ep \leq \eta M$, we have $\kappa < 2 \eta M \leq M/4$.
  Taking $\kappa' = M/4$ in \cref{lem:TwoLayer_AppBelowLargeKappa} and \cref{cor:TwoLayer_AppAboveLargeKappa},
  we have $\abs{\theta^*_k - \theta_k(t)} \leq 2\kappa' = M/2$ after the maximum of
  \begin{align*}
    \xk{\kappa'}^{-1} \zk{2 + \frac{1}{2} \ln^+ \frac{M}{\lambda_k} + \frac{1}{2}\ln^+ \frac{\abs{\theta_k^*}-2\kappa'}{\lambda_k}}
    \leq 4 M^{-1} \zk{2 + \frac{1}{2} \ln^+ \frac{M}{\lambda_k} + \frac{1}{2} \ln^+\frac{\abs{\theta_k^*}}{\lambda_k}}
  \end{align*}
  or
  \begin{align*}
    \xk{\kappa'}^{-1} \xk{1 + \ln^+ \frac{M}{\kappa'}} \leq 4 M^{-1} \ln 4 \leq 8 M^{-1}.
  \end{align*}
  It shows that we have
  \begin{align*}
    \overline{T}^{\mr{half}}(M) \leq 4 M^{-1} \zk{2 + \frac{1}{2} \max_{k \in S} \xk{\ln^+ \frac{M}{\lambda_k}+ \ln^+\frac{\abs{\theta^*_k}}{\lambda_k}}}.
  \end{align*}

  If $\ep > \eta M$, we have $\kappa < 2\ep$.
  Taking $\kappa' = 2\ep$,
  when $\abs{\theta^*_k} \geq 2\kappa'$, the result then follows from \cref{cor:TwoLayer_AppBoundSmallKappa} where we take $\delta = 2\ep$.
  For the last statement, we apply \cref{cor:TwoLayer_Rectracting}.
\end{proof}

With the previous proposition, we can consider the following iterative shrinkage dynamics.
Let us choose $\nu_1 = \ep / \eta = C n^{-1/2+s_1} \sqrt{\ln n}$.
Define  $M_i = 2^{-i} B_{\infty}$ for $i = 0,1,\dots,I+1$ with $I = \lfloor \log_2 (B_{\infty}/\nu_1) \rfloor$.
Then it is clear that $I \leq \log_2 \sqrt {n} + C$.
Since $\norm{\bm\theta^*_S - \bm\theta_S(0)}_{\infty} = \norm{\theta^*_S}_{\infty} \leq B_{\infty}$,
we use \cref{prop:TwoLayer_Shrinkage} iteratively to obtain that
\begin{align*}
  \norm{\bm\theta^*_S - \bm\theta_S(t)}_{\infty} \leq M_{i+1},\quad \forall t \geq \sum_{j=0}^i T^{(j)},
\end{align*}
where
\begin{align*}
  T^{(i)} = 4 M_i^{-1} \zk{2 + \frac{1}{2} \max_{k \in S} \xk{\ln^+ \frac{M_i}{\lambda_k}+ \ln^+\frac{\abs{\theta^*_k}}{\lambda_k}}} \leq C M_i^{-1} \ln n,
\end{align*}
since $\lambda_k \asymp k^{-\gamma}$ and $\max S \leq C n^{\kappa/2}$.
Now, we have
\begin{align}
  \overline{T} \coloneqq \sum_{i=0}^I T^{(i)} \leq C \sum_{i=0}^I M_i^{-1} \ln n \leq C \nu_1^{-1} \ln n \leq C n^{1/2-s_1} \ln n \leq C n^{1/2}.
\end{align}

Finally, when $t \geq \overline{T}$, we have $\eta M_{I+1} < \ep$.
Choosing $\nu_2 = n^{-1/2} \ln n$, we use the second part of \cref{prop:TwoLayer_Shrinkage} to obtain that
for all $k \in J_{\mr{sig}}(\nu_2)$,
\begin{align*}
  \abs{\theta_k^* - \theta_k(t)} \leq 4 \ep,\quad \forall t \geq \overline{T} + T^{\mr{fin}},
\end{align*}
where
\begin{align*}
  T^{\mr{fin}} &\leq C \max_{k \in J_{\mr{sig}}(\nu_2)} 4\abs{\theta^*_k}^{-1}
  \zk{2 + \frac{1}{2}\ln^+\frac{M_{I+1}}{\lambda_k} + \frac{1}{2}\ln^+ \frac{\abs{\theta^*_k}}{4\lambda_k} + \ln^+ \frac{\abs{\theta^*_k}}{2\ep}} \\
  &\leq C \nu_2^{-1} \ln n \leq C n^{1/2}.
\end{align*}
On the other hand, for $k \in S \backslash J_{\mr{sig}}(\nu_2)$, we use \cref{eq:EqTwoLayer_ShrinkageBound_FinalSmallSig} to get
\begin{align*}
  \abs{\theta^*_k - \theta_k(t)} \leq 2\max(\abs{\theta^*_k},4\ep)
\end{align*}
provided $t \geq \overline{T} + C n^{1/2} \geq \overline{T} + (2\ep)^{-1}$.
Recalling the choice of $S$ in \cref{eq:Proof_ComponentPartition}, we find that
\begin{align*}
    \sum_{k \in S \backslash J_{\mr{sig}}(\nu_2)} \abs{\theta^*_k - \theta_k(t)}^2 &= \sum_{k \in S_1 \backslash J_{\mr{sig}}(\nu_2)} \abs{\theta^*_k - \theta_k(t)}^2 + \sum_{k \in S_2 \backslash J_{\mr{sig}}(\nu_2)} \abs{\theta^*_k - \theta_k(t)}^2\\
    & \leq C \sum_{k \in S_1 \backslash J_{\mr{sig}}(\nu_2)} \abs{\theta_k^*}^2 
    + C \sum_{k \in S_2 \backslash (S_1 \cup J_{\mr{sig}}(\nu_2))} \ep^2 \\
    & \leq C \Psi(\nu_2) + C n^{\frac{1}{2\gamma}} \ep^2
\end{align*}

Consequently, we conclude that after $t \geq \underline{C} n^{1/2}$, we have
\begin{align}
  \notag
  \norm{\bm\theta^*_S - \bm\theta_S(t)}_2^2
  &= \sum_{k \in J_{\mr{sig}}(\nu_2)} \abs{\theta^*_k - \theta_k(t)}^2
  + \sum_{k \in S \backslash J_{\mr{sig}}(\nu_2)} \abs{\theta^*_k - \theta_k(t)}^2 \\
  \notag
  &\leq C \ep^2 \abs{S} + C \Psi(\nu_2) + C n^{\frac{1}{2\gamma}} \ep^2 \\
  \label{eq:Proof_TwoLayer_GenSignal}
  &= C \frac{\ln n}{n} \xk{\Phi\xkm{n^{-1/2}\sqrt {\ln n}} + n^{\frac{1}{2\gamma}}} + C \Psi\xkm{n^{-1/2}\ln n}.
\end{align}
Here, we note that the constant $\underline{C}$ does not depend on $E$.

\subsubsection{Bounding the Noise}
\label{subsubsec:Proof_Noise}
In this part, we will use \cref{lem:TwoLayer_ErrorControl}
to show that $T^{\mr{err}} \geq \bar{C} \sqrt {n}$ for some constant $\bar{C} > 0$ if $E$ is chosen accordingly.
Since $\abs{\theta_k^*} \leq n^{-1/2} \sqrt {\ln n} $, it suffices to bound the term
$\int_{0}^{\bar{C}\sqrt {n}} \abs{h_k(t)}\dd t$,
where we recall that $h_k = p_k + q_k + r_k$.

Let us denote $\overline{T}^{(i)} = \sum_{j=0}^i \overline{T}^{(j)}$.
Now, for $t \in [\overline{T}^{(i)}, \overline{T}^{(i+1)}]$, using the results in the previous part, we have
$\norm{\bm\theta^*_S - \bm\theta_S(t)}_{\infty} \leq M_i$, so the first bound in \cref{eq:Proof_BasicControls} gives
\begin{align*}
  \abs{p_k(t)} \leq C M_i n^{-s} \sqrt {\ln k},\quad t \in [\overline{T}^{(i)}, \overline{T}^{(i+1)}].
\end{align*}
Thus,
\begin{align*}
  \int_{0}^{\overline{T}} \abs{p_k(t)}\dd t &\leq C \sum_{i=0}^I M_i n^{-s} \sqrt {\ln k} \cdot T^{(i)}
  \leq C \sum_{i=0}^I M_i n^{-s} \sqrt {\ln k} \cdot M_i^{-1} \ln n \\
  &\leq \abs{I} n^{-s} \sqrt {\ln k} \ln n \leq C \sqrt {\ln k}.
\end{align*}
Moreover, for $t \geq \overline{T}$, we have
\begin{align*}
  \abs{p_k(t)} \leq \eta_k \norm{(\bm\theta^* - \bm\theta)_S}_{\infty} \leq C n^{-s} \sqrt {\ln (nk)} \cdot M_{I+1}
  \leq C n^{-1/2} \sqrt {\ln k},
\end{align*}
so
\begin{align*}
  \int_{\overline{T}}^{\bar{C}\sqrt {n}} \abs{p_k(t)}\dd t \leq C \bar{C} \sqrt {\ln k}.
\end{align*}
On the other hand, the other two bounds in \cref{eq:Proof_BasicControls} gives $\abs{q_k + r_k} \leq C n^{-1/2} \sqrt {\ln n + \ln k}$,
so
\begin{align*}
  \int_{0}^{\bar{C}\sqrt {n}} \abs{q_k(t) + r_k(t)}\dd t \leq C \bar{C} \sqrt {\ln n + \ln k}.
\end{align*}
Combining the three bounds, we prove the claim.

\subsubsection{Generalization Error}
From \cref{eq:Proof_TwoLayer_GenSignal} and \cref{subsubsec:Proof_Noise},
we can choose a constant $\bar{C} \geq \underline{C}$ and then fix the constant $E$.
Now, when $\underline{C} \sqrt {n} \leq t \leq \bar{C} \sqrt {n} \leq T^{\mr{err}}$,
we have
\begin{align*}
  \norm{\bm{\theta}_R(t)}_2^2 & \leq C \sum_{k \in R} \lambda_k^2 \exp(2E \sqrt {\ln n + \ln k}) \\
  & \leq C \exp(2E \sqrt {\ln n}) \sum_{k \in R} \lambda_k^2 \exp(2E \sqrt {\ln k})\\
  & \leq C \exp(2E \sqrt {\ln n})  \sum_{k \geq L} \lambda_k^2 \exp(2E \sqrt {\ln k}) \\
  & \leq C \exp(2E \sqrt {\ln n}) L^{1-2\gamma+s} \leq C n^{-1+\frac{1+s}{2\gamma}}.
\end{align*}
Consequently,
\begin{align}
  \label{eq:Proof_TwoLayer_GenNoise}
  \norm{\bm{\theta}^*_R - \bm{\theta}_R(t)}_2^2
  \leq 2 \norm{\bm{\theta}^*_R}_2^2 + 2 \norm{\bm{\theta}_R(t)}_2^2
  \leq \Psi(n^{-1/2}\sqrt {\ln n}) + C n^{-1+\frac{1+s}{2\gamma}}.
\end{align}

Combining \cref{eq:Proof_TwoLayer_GenSignal} and \cref{eq:Proof_TwoLayer_GenNoise}, we obtain the desired result:
\begin{align}
  \norm{f^* - \hat{f}_t^{\mr{EigGD}}}_{L^2}^2 = \norm{\bm\theta^* - \bm\theta(t)}_2^2
  \lesssim \frac{1}{n}\zk{ \Phi\xkm{n^{-1/2}\sqrt {\ln n}} \ln n + n^{\frac{1+s}{2\gamma}}} + \Psi\xkm{n^{-1/2} \ln n}.
\end{align}



\subsection{One-dimensional Dynamics}\label{subsec:Proof_TwoLayer_OneDim}

In this subsection, we will focus on the one-dimensional dynamics of the gradient flow with perturbation.
Let $h(t): \R_{\geq 0} \to \R$ be a continuous perturbation function.
Let us consider now the following perturbed one-dimensional dynamics:
\begin{align}
  \label{eq:TwoLayer_1d_perturbed}
  \left\{
  \begin{aligned}
    \dot{\beta}(t) &= a(t) \zk{z - \theta(t) + h(t)},\quad \beta(0) = 0, \\
    \dot{a}(t) &= \beta(t) \zk{z - \theta(t) + h(t)}, \quad a(0) = \lambda^{1/2},
  \end{aligned}
  \right.
\end{align}
where $\theta(t) = a(t) \beta(t)$.

\noindent\textit{Conservation Quantities.}
Computing the time derivative of $a^2(t)$ and $\beta^2(t)$, we have
\begin{align*}
  \dv{t} a^2(t) = \dv{t} \beta^2(t) = 2 a(t) \beta(t) \zk{z - \theta(t) + h(t)},
\end{align*}
so we have the conservation quantity that
\begin{align}
  \label{eq:EqTwoLayerP_Conservation}
  a^2(t) - \beta^2(t) = a^2(0) - \beta^2(0) = \lambda,
\end{align}
so $a(t) = (\lambda + \beta^2(t))^{1/2}$ and thus
\begin{gather}
  \label{eq:TwoLayer_BoundA}
  \max(\lambda^{\hf},\abs{\beta(t)}) \leq a(t) \leq \sqrt{2}\max(\lambda^{\hf},\abs{\beta(t)}), \\
  \label{eq:TwoLayer_BoundTheta}
  \max(\lambda^{\hf},\abs{\beta(t)}) \abs{\beta(t)} \leq \abs{\theta(t)} \leq \sqrt{2}\max(\lambda^{\hf},\abs{\beta(t)}) \abs{\beta(t)}.
\end{gather}

\noindent\textit{The evolution of $\theta$.}
It is direct to compute that
\begin{align}
  \label{eq:TwoLayerEq_Theta}
  \begin{aligned}
    \dot{\theta} &= \dot{a} \beta + a \dot{\beta}
    = \zk{\beta^2 + a^2} (z - \theta + h) \\
    &= \theta^2 (a^{-2} + \beta^{-2}) (z - \theta + h).
  \end{aligned}
\end{align}

\begin{lemma}[Perturbation bound]
  \label{lem:TwoLayer_Monotonicity}
  Consider \cref{eq:TwoLayer_1d_perturbed}.
  Let $t_0 \geq 0$.
  Suppose there exists some $\kappa > 0$ such that $\abs{h(t)} \leq \kappa$ for all $t \geq t_0$.
  Then, for $t \geq t_0$,
  \begin{enumerate}[(i)]
    \item $\abs{z - \theta(t)}$ is decreasing provided that $\abs{z - \theta(t)} \geq \kappa$.
    \item Once $\abs{z - \theta(t_1)} \leq \kappa$ for some $t_1 \geq t_0$,
    we have $\abs{z - \theta(t)} \leq \kappa$ for all $t \geq t_1$.
  \end{enumerate}
\end{lemma}
\begin{proof}
  Without loss of generality, we assume that $z \geq 0$.
  We recall in \cref{eq:TwoLayerEq_Theta} that
  \begin{align*}
    \dot{\theta} = \theta^2 (a^{-2} + \beta^{-2}) (z - \theta + h).
  \end{align*}
  Now, if $\abs{z - \theta(t)} \geq \kappa$, we have either $\theta(t) \leq z - \kappa$ or $\theta(t) \geq z + \kappa$.
  For the first case, we have $\dot{\theta} \geq 0$ so $\theta(t)$ is increasing, which implies that $\abs{z - \theta(t)} = z - \theta(t)$ is decreasing.
  Similarly, for the second case, we have $\dot{\theta} \leq 0$ and thus $\abs{z - \theta(t)}$ is also decreasing.
  Consequently, we have the first part of the proposition.
  The second part is just a consequence of the first part.
\end{proof}

\begin{lemma}[Approaching from below, large $\kappa$]
  \label{lem:TwoLayer_AppBelowLargeKappa}
  Consider the equation \cref{eq:TwoLayer_1d_perturbed} with $z \geq 0$ (a similar result holds for $z \leq 0$).
  Let $t_0 \geq 0$.
  Suppose that there exists some $\kappa > 0$ such that $\abs{h(t)} \leq \kappa,~ \forall t \geq t_0$ and
  there exists some $M\geq 0$ such that
  \begin{align*}
    z - M \leq \theta(t_0) \leq z.
  \end{align*}
  Then, we have
  \begin{align*}
    \abs{z - \theta(t)} \leq 2\kappa,\quad \forall t \geq t_0 + \overline{T}^{\mr{app}},
  \end{align*}
  where
  \begin{align}
    \overline{T}^{\mr{app}} \leq \kappa^{-1} \zk{2 + \frac{1}{2} \ln^+ \frac{M}{\lambda} + \frac{1}{2}\ln^+ \frac{z-2\kappa}{\lambda}}.
  \end{align}
\end{lemma}
\begin{proof}

  From \cref{lem:TwoLayer_Monotonicity}, we see that $\abs{z - \theta(t)}$ is decreasing until it reaches $\kappa$
  and then it stays below $\kappa$.
  Therefore, it suffices to bound the hitting time of $\theta(t) = z-2\kappa$.
  Let us further define
  \begin{align*}
    T^{\mr{neg}} &= \inf \dk{ s\geq 0 : \beta(t_0 + s) \geq -\lambda^{1/2}},\quad
    T^{\mr{esc}} = \inf \dk{ s\geq 0 : \beta(t_0+s) \geq \lambda^{1/2}}\\
    T^{\mr{app}} &= \inf \dk{ s\geq 0 : \abs{z -\theta(t_0+s)} \leq 2\kappa}.
  \end{align*}
  We also recall that
  \begin{align}
    \label{eq:Proof_TwoLayer_AppBound}
    \dot{\beta}(t) = a(t) \zk{z - \theta(t) + h(t)} \geq \kappa a(t) ,\qq{for} t \leq T^{\mr{app}}.
  \end{align}
  In the following, let us suppose that $0 < T^{\mr{neg}} < T^{\mr{esc}} < T^{\mr{app}}$ and derive an upper bound for $T^{\mr{app}}$.
  It is easy to see that the resulting bound is also valid for the other cases (such as $T^{\mr{esc}} = 0$ or $T^{\mr{app}} \leq T^{\mr{esc}}$).

  \noindent \textit{Stage 1: Adjusting sign}\quad
  We have $\beta(t_0) \leq -\lambda^{1/2}$ and thus \cref{eq:TwoLayer_BoundA} gives
  \begin{align*}
    z - M \leq \theta(t_0) \leq -\sqrt {2}\beta^2(t_0),\quad \Longrightarrow \beta(t_0) \geq -\sqrt{M}.
  \end{align*}
  Also, we combine \cref{eq:Proof_TwoLayer_AppBound} and \cref{eq:TwoLayer_BoundA} to get
  \begin{align*}
    \dot{\beta}(t_0 +s) \geq -\kappa \beta(t_0 +s), \qq{for} s \in [0,T^{\mr{neg}}],
  \end{align*}
  so
  \begin{align*}
    \beta(t_0 + s) \geq \beta(t_0) \exp(-\kappa s) \geq -\sqrt {M} \exp(-\kappa s), \qq{for} s \in [0,T^{\mr{neg}}],
  \end{align*}
  and thus
  \begin{align*}
    T^{\mr{neg}} \leq \frac{1}{2} \kappa^{-1} \ln \frac{M}{\lambda}.
  \end{align*}

  \noindent \textit{Stage 2: Escaping the origin}\quad
  Now, we consider $t \geq T^{\mr{neg}}$.
  We plug $a \geq \lambda^{1/2}$ in \cref{eq:Proof_TwoLayer_AppBound} to derive
  \begin{align*}
    \dot{\beta} \geq \lambda^{\hf} \kappa,
  \end{align*}
  which implies that
  \begin{align*}
    T^{\mr{esc}} - T^{\mr{neg}} \leq 2 \kappa^{-1}.
  \end{align*}

  \noindent \textit{Stage 3: Approximating $z$}\quad
  Now, let us consider $t \geq T^{\mr{esc}}$.
  We now use $a \geq \beta$ to get
  \begin{align*}
    \dot{\beta} \geq \kappa \beta,
  \end{align*}
  which implies that
  \begin{align*}
    \beta(t_0 + T^{\mr{esc}} + s) \geq \lambda^{1/2} \exp(\kappa s), \quad \forall s \in [0, T^{\mr{app}} - T^{\mr{esc}}].
  \end{align*}
  Noticing that $\theta = a \beta \geq \beta^2$, we have
  \begin{align*}
    \beta \geq \sqrt {z-2\kappa} \quad \Longrightarrow \quad \theta \geq z - 2\kappa,
  \end{align*}
  so
  \begin{align*}
    T^{\mr{app}} - T^{\mr{esc}} \leq \kappa^{-1} \ln \frac{z-2\kappa}{\lambda^{\hf}}.
  \end{align*}

  Summing up, we have
  \begin{align*}
    T^{\mr{esc}} \leq \kappa^{-1} \zk{2 + \frac{1}{2} \ln^+ \frac{M}{\lambda} + \frac{1}{2}\ln^+ \frac{z-2\kappa}{\lambda}}.
  \end{align*}
\end{proof}

\begin{lemma}[Signal time from below, small $\kappa$]
  \label{lem:TwoLayer_AppBelowSmallKappa}
  Consider the same setting as in \cref{lem:TwoLayer_AppBelowLargeKappa}.
  Assume further that $z \geq 2\kappa$.
  Then,
  \begin{align*}
    \theta(t) \geq z/4,\quad \forall t \geq t_0 + \overline{T}^{\mr{sig}},
  \end{align*}
  where
  \begin{align*}
    \overline{T}^{\mr{sig}} = 4z^{-1}\zk{2 + \frac{1}{2}\ln^+\frac{M}{\lambda} + \frac{1}{2}\ln^+ \frac{z}{4\lambda}}.
  \end{align*}
\end{lemma}
\begin{proof}
  The proof resembles that of \cref{lem:TwoLayer_AppBelowLargeKappa} but we use a different bound for the term:
  \begin{align*}
    z - \theta(t) + h(t) \geq \frac{3}{4}z - \kappa \geq \frac{1}{4}z,\qq{for} t \in [t_0, T^{\mr{sig}}],
  \end{align*}
  where
  \begin{align*}
    T^{\mr{sig}} = \inf \dk{ s\geq 0 : \theta(t_0+s) \geq z/4}.
  \end{align*}
  Following the proof of \cref{lem:TwoLayer_AppBelowLargeKappa}, we have
  \begin{align*}
    T^{\mr{neg}} &\leq 2 z^{-1} \ln \frac{M}{\lambda},\quad
    T^{\mr{esc}} - T^{\mr{neg}} \leq 8z^{-1},\quad
    T^{\mr{sig}} - T^{\mr{esc}} \leq 2z^{-1} \ln^+ \frac{z}{4\lambda},
  \end{align*}
  and the desired result follows.
\end{proof}

\begin{lemma}[Approaching from above]
  \label{lem:TwoLayer_AppAbove}
  Consider the equation \cref{eq:TwoLayer_1d_perturbed} with $z \geq 0$ (a similar result holds for $z \leq 0$).
  Let $t_0 \geq 0$.
  Suppose that there exists some $\kappa > 0$ such that $\abs{h(t)} \leq \kappa,~ \forall t \geq t_0$ and $\theta(t_0) \geq z$.
  Then, we have
  \begin{align*}
    \abs{z - \theta(t)} \leq 2\max(z,2\kappa),\quad \forall t \geq t_0 + \max(z,2\kappa)^{-1}.
  \end{align*}

\end{lemma}
\begin{proof}
  Similar to the proof of \cref{lem:TwoLayer_AppBelowLargeKappa},
  from \cref{lem:TwoLayer_Monotonicity}, it suffices to consider the case $M > 2\max(z,2\kappa)$ and bound the hitting time of $\theta(t_0 + s) = z+2\max(z,2\kappa)$,
  which is denoted by $T^{\mr{sig}}$.

  Recall that \cref{eq:TwoLayerEq_Theta} and the fact that $\theta \geq \beta^2$ give
  \begin{align*}
    \dot{\theta} = \theta^2 (a^{-2} + \beta^{-2}) (z - \theta + h)
    \leq -\theta (\theta - z - \kappa).
  \end{align*}
  Now, noticing that before $T^{\mr{sig}}$, we have
  \begin{align*}
    \frac{1}{2}\theta \geq \frac{1}{2} z + \max(z,2\kappa) \geq z + \kappa,
  \end{align*}
  and thus $\theta - z - \kappa \geq \frac{1}{2}\theta$.
  This gives
  \begin{align*}
    \dot{\theta}(t_0+s) \leq - \frac{1}{2} \theta^2(t_0+s),\qq{for} s \in [0,T^{\mr{app}}],
  \end{align*}
  which implies that
  \begin{align*}
    T^{\mr{sig}} \leq 2 \xk{\xk{z+2\max(z,2\kappa)}^{-1} - (z+M)^{-1}} \leq
    \max(z,2\kappa)^{-1}.
  \end{align*}
\end{proof}

\begin{lemma}[Approximation time near $z$]
  \label{lem:TwoLayer_FinalTime}
  Consider the equation \cref{eq:TwoLayer_1d_perturbed} with $z \geq 0$ (a similar result holds for $z \leq 0$).
  Let $t_0 \geq 0$.
  Suppose that there exists some $\kappa > 0$ such that $\abs{h(t)} \leq \kappa,~ \forall t \geq t_0$.
  Suppose also that
  \begin{align*}
    \frac{1}{4}z \leq \theta(t_0) \leq 3 z.
  \end{align*}
  Then, for any $\delta > 0$, we have
  \begin{align*}
    \abs{z - \theta(t)} \leq \kappa + \delta,\quad \forall t \geq t_0 + 4 z^{-1}  \ln^+ \frac{\abs{z - \theta(t_0)} - \kappa}{\delta}.
  \end{align*}
\end{lemma}
\begin{proof}
  Let us define
  \begin{align*}
    T^{\mr{app}} = \inf \dk{ s\geq 0 : \abs{z- \theta(t_0+s)} \leq \kappa+\delta}.
  \end{align*}
  Using \cref{lem:TwoLayer_Monotonicity}, it suffices to prove an upper bound of $T^{\mr{app}}$ when it is not zero,
  so from the monotonicity, we have $\frac{1}{4}z \leq \theta(t) \leq 3 z$ for all $t \geq t_0$.

  Now, the lower bound $\beta^2 \leq \theta$ in \cref{eq:TwoLayer_BoundTheta} yields $\beta^{-2} \geq \theta^{-1}$.
  Plugging these into \cref{eq:TwoLayerEq_Theta}, we have
  \begin{align*}
    \theta^2 (a^{-2} + \beta^{-2}) \geq \theta \geq z/4,
  \end{align*}
  and thus
  \begin{align*}
    \dot{\theta} &\geq \frac{1}{4}z (z-\kappa - \theta),\qq{if} \theta(t_0) \leq z - \kappa,\\
    \dot{\theta} &\leq -\frac{1}{4}z (\theta - z-\kappa),\qq{if} \theta(t_0) \geq z + \kappa.
  \end{align*}
  It shows that for $s \in [0,T^{\mr{app}}]$,
  \begin{align*}
    z - \kappa - \theta(t_0 + s) &\leq (z - \kappa - \theta(t_0)) \exp(-\frac{1}{4}zs),\qq{if} \theta(t_0) \leq z - \kappa, \\
    \theta(t_0 + s) - z - \kappa &\leq (\theta(t_0) - z - \kappa) \exp(-\frac{1}{4}zs),\qq{if} \theta(t_0) \geq z + \kappa,
  \end{align*}
  and thus
  \begin{align*}
    T^{\mr{app}} \leq 4z^{-1} \ln \frac{\abs{z - \theta(t_0)} - \kappa}{\delta}.
  \end{align*}
\end{proof}

\begin{corollary}
  \label{cor:TwoLayer_AppAboveLargeKappa}
  Under the same setting as in \cref{lem:TwoLayer_AppAbove}, we have
  \begin{align*}
    \abs{z - \theta(t)} \leq 4 \kappa,\quad \forall t \geq t_0 + \kappa^{-1}\xk{1 + \ln^+ \frac{\abs{z - \theta(t_0)}}{\kappa}}.
  \end{align*}
\end{corollary}
\begin{proof}
  \cref{lem:TwoLayer_AppAbove} yields that we have
  \begin{align*}
    \abs{z - \theta(t)} \leq 2\max(z,2\kappa),\quad \forall t \geq t_1 = t_0 + \max(z,2\kappa)^{-1}.
  \end{align*}
  Now, if $z \leq 2\kappa$ then we are done.
  Otherwise, we have $z \geq 2\kappa$ and thus we can further apply \cref{lem:TwoLayer_FinalTime} to get the desired result,
  where we note that the monotonicity also implies $\abs{z - \theta(t_1)} \leq \max(\abs{z - \theta(t_0)}, \kappa)$.

\end{proof}

\begin{corollary}
  \label{cor:TwoLayer_AppBoundSmallKappa}
  Consider the equation \cref{eq:TwoLayer_1d_perturbed} with $z \geq 0$ (a similar result holds for $z \leq 0$).
  Let $t_0 \geq 0$.
  Suppose that there exists some $\kappa > 0$ such that $\abs{h(t)} \leq \kappa,~ \forall t \geq t_0$ and
  $z \geq 2\kappa$.
  Then,
  \begin{align*}
    \abs{z - \theta(t)} \leq \kappa + \delta,\quad \forall t \geq t_0 + \overline{T}^{\mr{app}}(\delta),
  \end{align*}
  where
  \begin{align*}
    \overline{T}^{\mr{app}}(\delta) \leq 4z^{-1}\zk{2 + \frac{1}{2}\ln^+\frac{\abs{z-\theta(t_0)}}{\lambda} + \frac{1}{2}\ln^+ \frac{z}{4\lambda} + \ln^+ \frac{2z}{\delta}}.
  \end{align*}
\end{corollary}
\begin{proof}
  It is a consequence of \cref{lem:TwoLayer_AppBelowSmallKappa}, \cref{lem:TwoLayer_AppAbove} and \cref{lem:TwoLayer_FinalTime}.
\end{proof}

\begin{lemma}[Retracting from the negative]
  \label{lem:TwoLayer_RectractNeg}
  Consider the equation \cref{eq:TwoLayer_1d_perturbed} with $z \geq 0$ (a similar result holds for $z \leq 0$).
  Let $t_0 \geq 0$.
  Suppose that there exists some $\kappa > 0$ such that $\abs{h(t)} \leq \kappa,~ \forall t \geq t_0$.
  Then, we have
  \begin{align*}
    \theta(t) \geq - 2 \kappa,\quad \forall t \geq t_0 + \kappa^{-1}.
  \end{align*}
\end{lemma}
\begin{proof}
  The proof is similar to that of \cref{lem:TwoLayer_AppAbove}, but now we treat $\theta(t) \leq 0$.
  We first recall \cref{eq:TwoLayerEq_Theta}.
  The fact that $\abs{\theta} \geq \beta^2$ yields
  $\theta^2 (a^{-2} + \beta^{-2}) \geq \abs{\theta}$.
  When $\theta(t) \leq - 2\kappa$, we have $z - \theta(t) - \kappa \geq \theta(t)/2$, and thus
  \begin{align*}
    \dot{\theta} \geq \abs{\theta(t)} (z - \theta(t) - \kappa) \geq \frac{1}{2}\theta^2(t),
  \end{align*}
  which implies that
  \begin{align*}
    \theta(t_0 + s) \geq - \zk{s/2 - \theta(t_0)^{-1}}^{-1},
  \end{align*}
  and thus
  \begin{align*}
    \theta(t_0 + s) \geq -2\kappa,\quad \forall s \geq \kappa^{-1}.
  \end{align*}
\end{proof}

\begin{corollary}[Retracting into $\kappa$]
  \label{cor:TwoLayer_Rectracting}
  Consider the equation \cref{eq:TwoLayer_1d_perturbed} .
  Let $t_0 \geq 0$.
  Suppose that there exists some $\kappa > 0$ such that $\abs{h(t)} \leq \kappa,~ \forall t \geq t_0$.
  Then,
  \begin{align*}
    \abs{z - \theta(t)} \leq 2\max(\abs{z},2\kappa),\quad \forall t \geq t_0 + \kappa^{-1}.
  \end{align*}
\end{corollary}
\begin{proof}
  If $\theta(t_0) \geq z$, we use  \cref{lem:TwoLayer_AppAbove}.
  If $\theta(t_0) \leq 0$, we use \cref{lem:TwoLayer_RectractNeg} and note that $\abs{z - \theta(t)} \leq \abs{z}+\abs{\theta(t)}$.
\end{proof}

\begin{lemma}[Error control]
  \label{lem:TwoLayer_ErrorControl}
  Consider the equation \cref{eq:TwoLayer_1d_perturbed}.
  We have
  \begin{align}
    \label{eq:TwoLayer_ErrorControl_Beta}
      \abs{\beta(t)} &\leq \lambda^{1/2} \exp(\sqrt{2} \int_{0}^t \xk{\abs{h(s)} + \abs{z}}\dd s), \\
      \label{eq:TwoLayer_ErrorControl} 
      \abs{\theta(t)} & \leq \sqrt {2} \lambda \exp(2\sqrt{2} \int_{0}^t \xk{\abs{h(s)} + \abs{z}} \dd s).
  \end{align}
\end{lemma}
\begin{proof}
  Without loss of generality, we assume that $z \geq 0$.
  We recall the equation \cref{eq:TwoLayer_1d_perturbed} that
  \begin{align*}
    \dot{\beta}(t) = a(t) \zk{z - \theta(t) + h(t)}.
  \end{align*}
  To give a bound of $\abs{\beta}$, we notice that $a(t)$ is always positive and thus $\theta(t)$ has the same sign as $\beta(t)$.
  Now, if $\theta(t) \geq 0$, we have $\abs{z - \theta(t) + h(t)} \leq z + \abs{h(t)}$.
  On the other hand, if $\theta(t) \leq 0$ and thus $\beta(t) \leq 0$, we have $z - \theta(t) + h(t) \geq - \abs{h(t)}$.
%
  Therefore, we conclude that
  \begin{align*}
    \abs{\beta(t)} \leq \int_0^t a(s) \xk{z + \abs{h(t)}} \dd s.
  \end{align*}

  Let us now consider
  \begin{align*}
    T^{\mr{esc}} = \inf\dk{t \geq 0: \abs{\beta(t)} \geq \lambda^{1/2}}.
  \end{align*}
  Then, when $t \leq T^{\mr{esc}}$, \cref{eq:EqTwoLayerP_Conservation} implies that
  $a \leq \sqrt {2}\lambda^{1/2}$, so we have
  Consequently, we have
  \begin{align*}
    \abs{\beta(t)} \leq \sqrt {2} \lambda^{1/2} \int_0^t \xk{\abs{h(s)} + \abs{z}}\dd s, \quad t \in [0,T^{\mr{esc}}],
  \end{align*}
  and hence
  \begin{align*}
    T^{\mr{esc}} \geq \inf \dk{ t \geq 0: \sqrt {2} \int_0^t \xk{\abs{h(s)} + \abs{z}}\dd s = 1}.
  \end{align*}
  Now, after $T^{\mr{esc}}$, we apply $a \leq \sqrt {2}\abs{\beta}$ to get
  \begin{align*}
    \abs{\beta(t)} \leq \lambda^{1/2} + \sqrt {2} \int_{T^{\mr{esc}}}^t \abs{\beta(s)} \xk{z + \abs{h(t)}} \dd s, \quad t \geq T^{\mr{esc}},
  \end{align*}
  and thus
  \begin{align*}
    \abs{\beta(t)} \leq \lambda^{1/2} \exp(\sqrt{2} \int_{T^{\mr{esc}}}^t \xk{\abs{h(s)} + \abs{z}}\dd s), \quad t \geq T^{\mr{esc}}.
  \end{align*}
  Using $\theta = a \beta \leq \sqrt {2}\beta^2$, we get for $t \geq T^{\mr{esc}}$,
  \begin{align*}
    \abs{\theta(t)} \leq \sqrt {2} \lambda \exp(2\sqrt{2} \int_{T^{\mr{esc}}}^t \xk{\abs{h(s)} + \abs{z}} \dd s)
    \leq \sqrt {2} \lambda \exp(2\sqrt{2} \int_{0}^t \xk{\abs{h(s)} + \abs{z}} \dd s).
  \end{align*}
  Finally, we note that this bound is also valid for $t \leq T^{\mr{esc}}$ since at that time $\abs{\theta(t)} \leq \sqrt {2}\lambda^{1/2}$.

\end{proof}


\section{Numerical Simulations}\label{sec:numerical-experiments}

In this section, we provide some numerical simulations to validate the theoretical results.
Let us focus the setting in \cref{example:low-dim-structure} where $\bm{x} \sim \mr{Unif}([-1,1)^d)$ and the eigenfunctions can be chosen as the trigonometric functions.
For training, we discretize the gradient flow by gradient descent with a small step size.

To illustrate the training dynamics of the adaptive kernel method, we consider the $d=2$ case.
We choose the truth function as $f^*(\bm{x}) = \cos(7.5 \pi x_1)$, where $\bm{x} = (x_1,x_2)$,
Then, it can be computed that the coefficients of $f^*$ over the eigen-functions are given by
$\ang{f^*, \cos(k\pi x_1)} = \frac{30}{(4k^2 -225) \pi}$,
$\ang{f^*, \sin(k\pi x_1)} = 0$.
We choose the eigenvalues as $\lambda_{\bm{m}} = (1+\norm{\bm{m}}_2^2)^{-3/2}$.
With $n=400$ samples, we plot the generalization error curve and the evolution of the coefficients in \cref{fig:ComparisonShort}.
From the plot of generalization error, we can see that the adaptive kernel method outperforms the fixed kernel method.
Moreover, from the plot of coefficients,
we can observe that the fixed kernel method learns noise in unrelated directions, while the adaptive kernel method focuses on the true direction,
which demonstrates the advantage of the adaptive kernel method in learning the signal's structure.
This is also in line with the theoretical result in \cref{prop:EigLearn}.

\begin{figure}[hp]
  \centering

  \subfigure{
    \begin{minipage}[t]{0.5\linewidth}
      \centering
      \includegraphics[width=1.\linewidth]{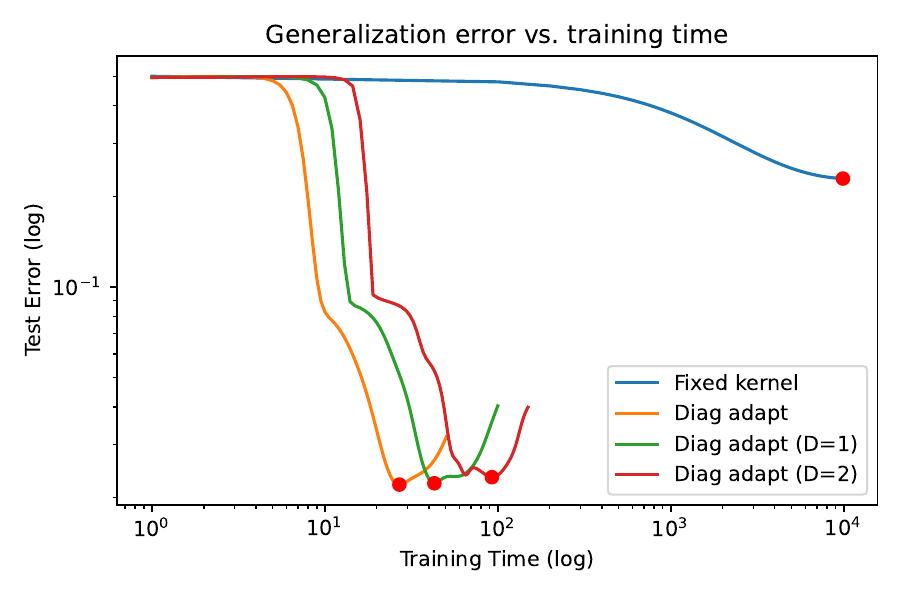}
    \end{minipage}%
  }%
  \subfigure{
    \begin{minipage}[t]{0.5\linewidth}
      \centering
      \includegraphics[width=1.\linewidth]{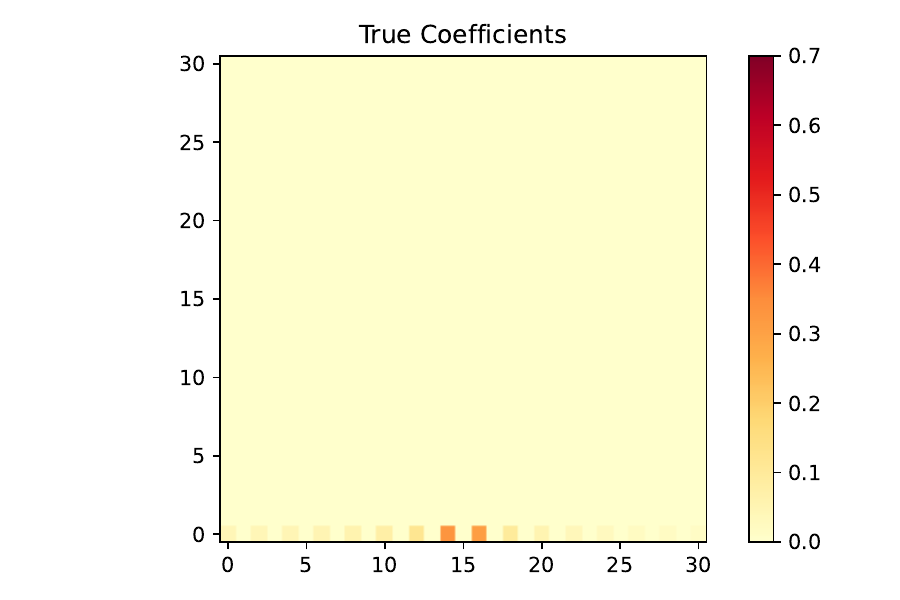}
    \end{minipage}%
  }%

  \subfigure{
    \begin{minipage}[t]{1\linewidth}
      \centering
      \includegraphics[width=1.\linewidth]{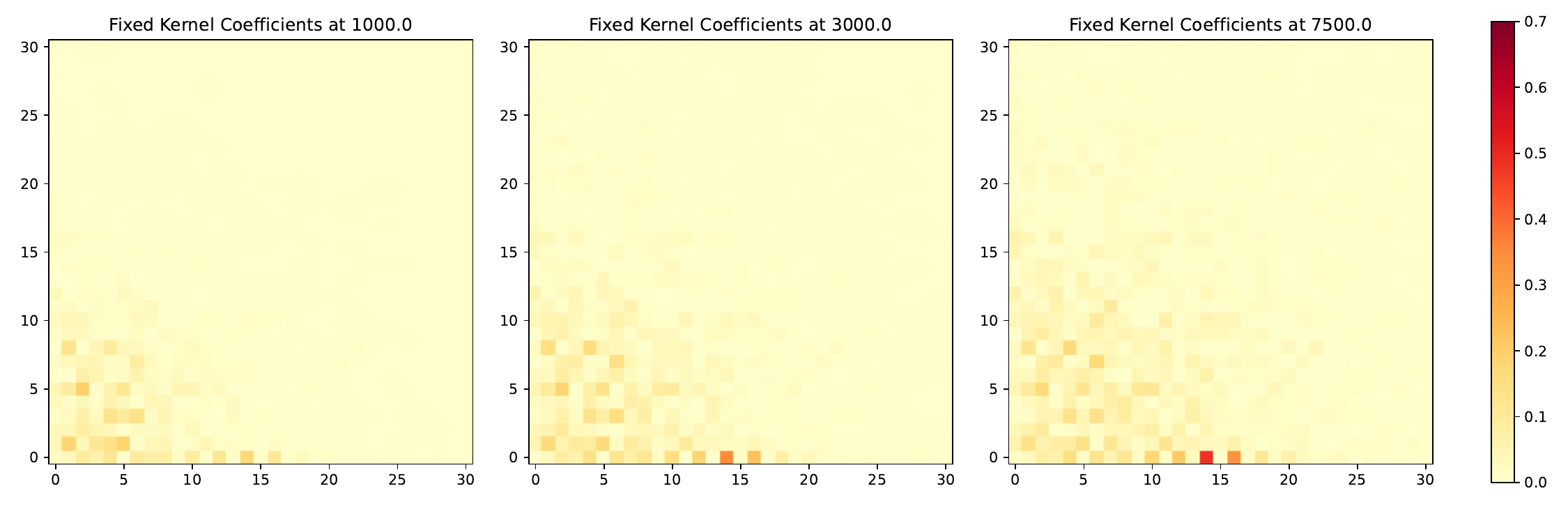}
    \end{minipage}%
  }%

  \subfigure{
    \begin{minipage}[t]{1\linewidth}
      \centering
      \includegraphics[width=1.\linewidth]{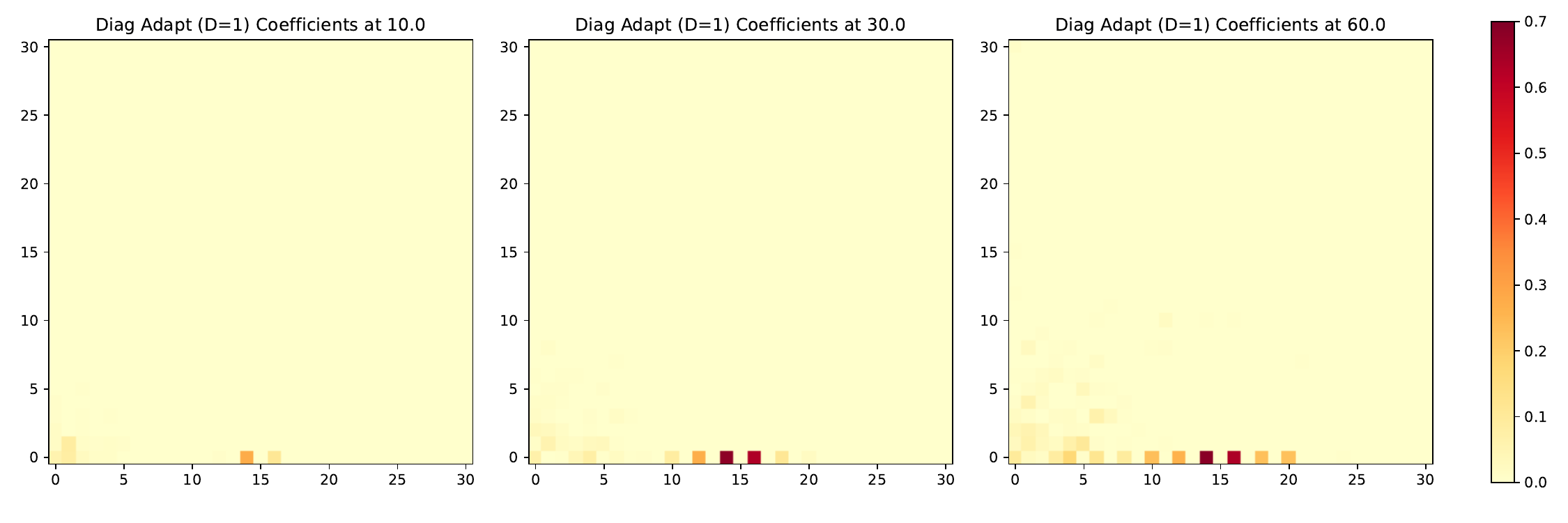}
    \end{minipage}%
  }%

  \centering
  \caption{
    Comparison of the fixed kernel method and the diagonal adaptive kernel method over a $d=2$ example with low-dimensional structure.
    The upper left figure shows the generalization error curve.
    The upper right figure shows the true coefficients and the lower two rows show the evolution of the coefficients for the fixed kernel method and the diagonal adaptive kernel method.
    Since the indices of the eigenfunctions are 2-dimensional, we plot the coefficients in a 2D grid.
  }

  \label{fig:ComparisonShort}
\end{figure}

We conduct further experiments to compare the generalization error of the fixed kernel method and the diagonal adaptive kernel method
in various settings.
See \cref{fig:ComparisonRate}.
The results show that the adaptive kernel method can achieve significantly better generalization error than the fixed kernel method.
Also, we can observe the benefit of introducing additional layers in the adaptive kernel method.
Overall, the numerical simulations support our theoretical results.

\begin{figure}[hp]
  \centering

  \subfigure{
    \begin{minipage}[t]{0.33\linewidth}
      \centering
      \includegraphics[width=1.\linewidth]{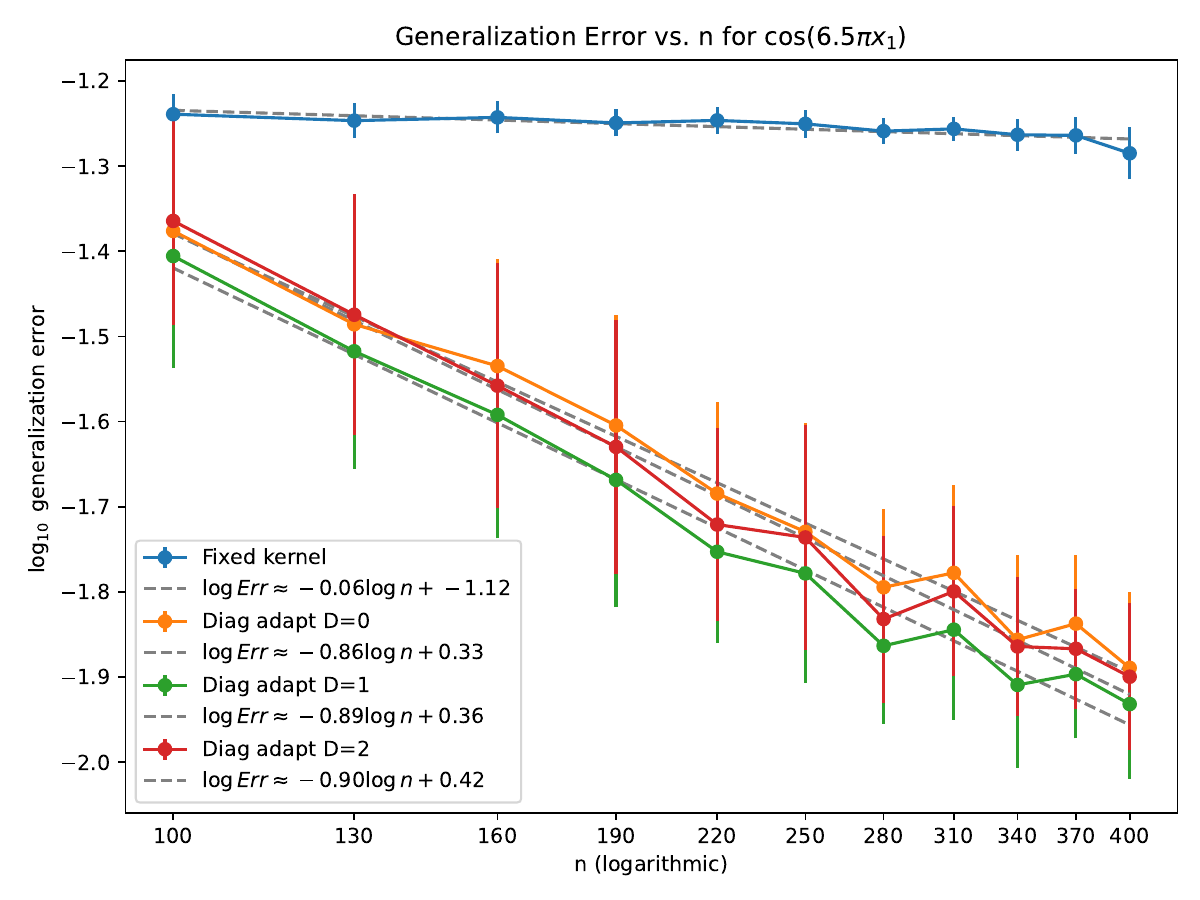}
    \end{minipage}%
  }%
  \subfigure{
    \begin{minipage}[t]{0.33\linewidth}
      \centering
      \includegraphics[width=1.\linewidth]{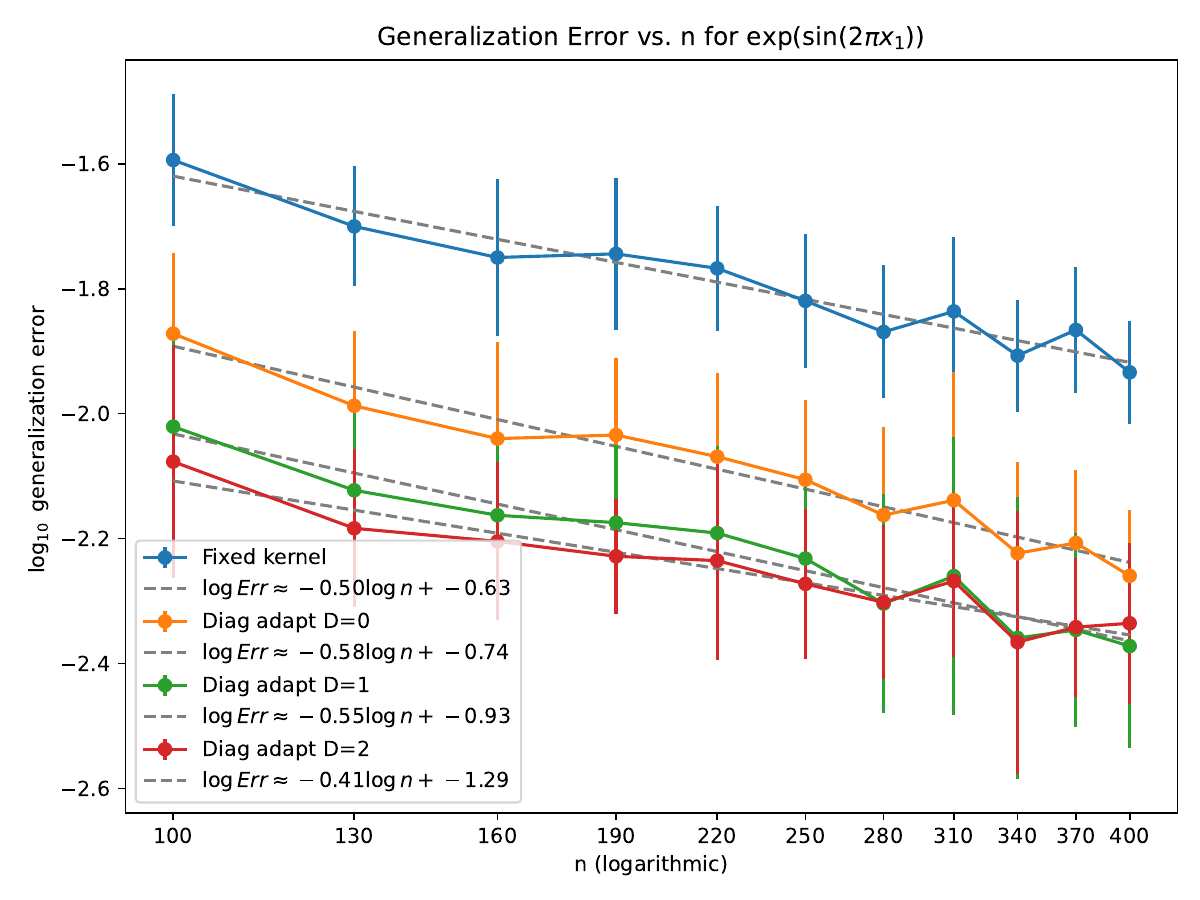}
    \end{minipage}%
  }%
  \subfigure{
    \begin{minipage}[t]{0.33\linewidth}
      \centering
      \includegraphics[width=1.\linewidth]{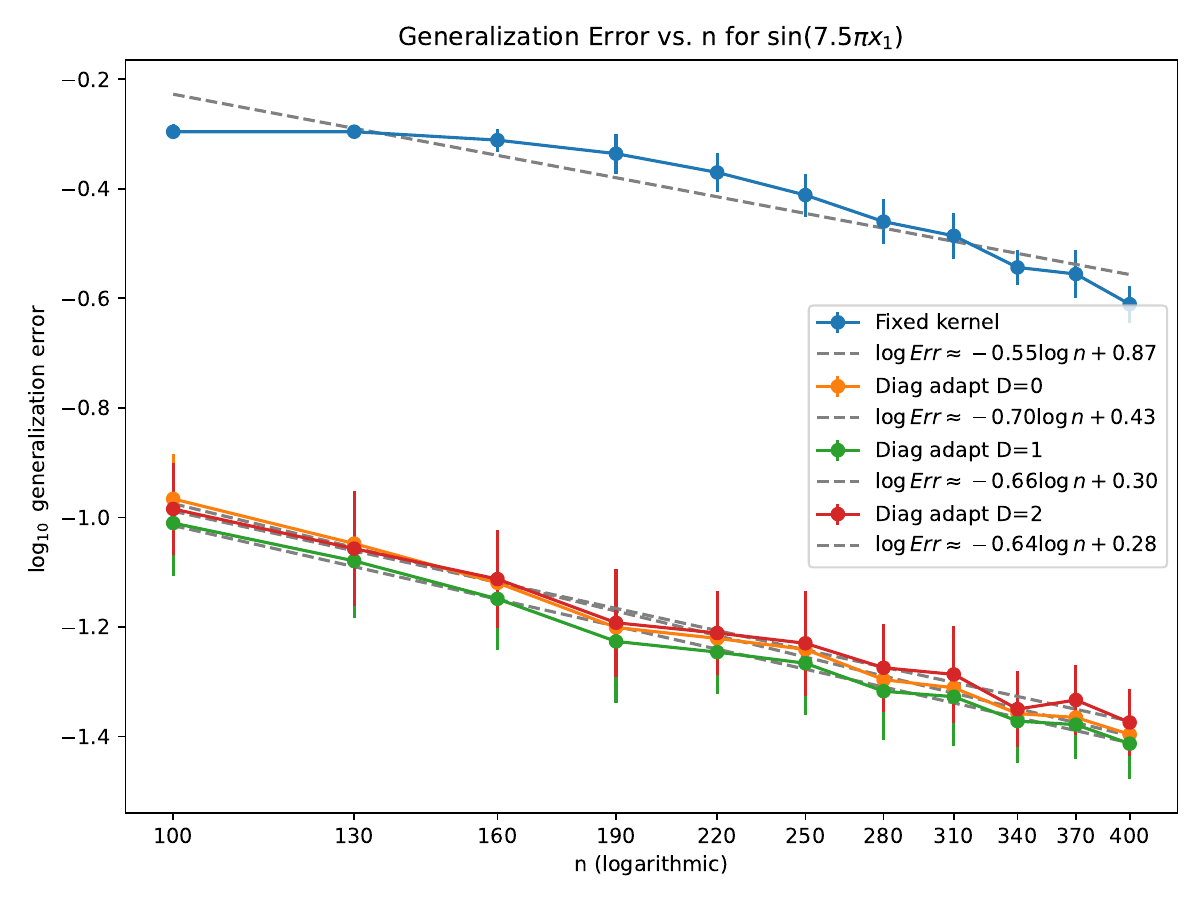}
    \end{minipage}%
  }%

  \subfigure{
    \begin{minipage}[t]{0.33\linewidth}
      \centering
      \includegraphics[width=1.\linewidth]{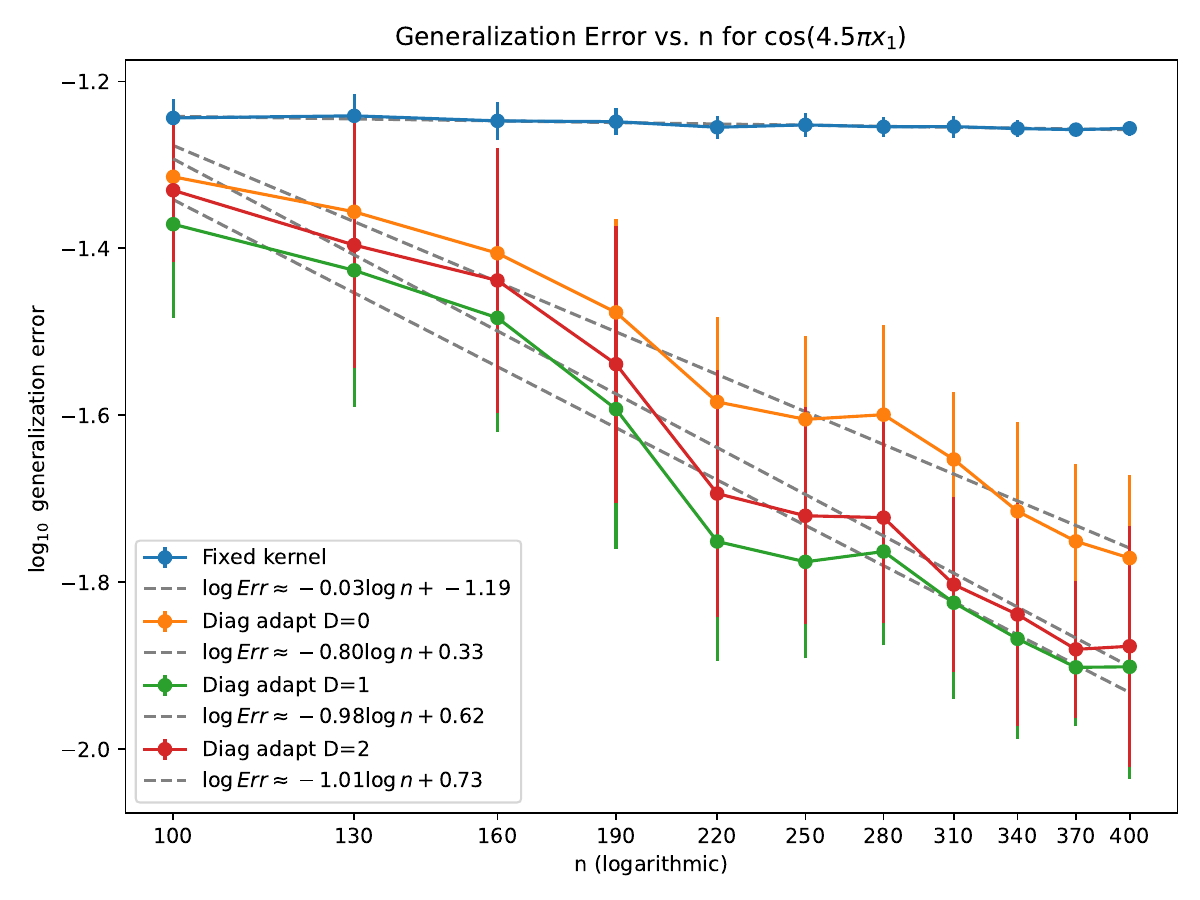}
    \end{minipage}%
  }%
  \subfigure{
    \begin{minipage}[t]{0.33\linewidth}
      \centering
      \includegraphics[width=1.\linewidth]{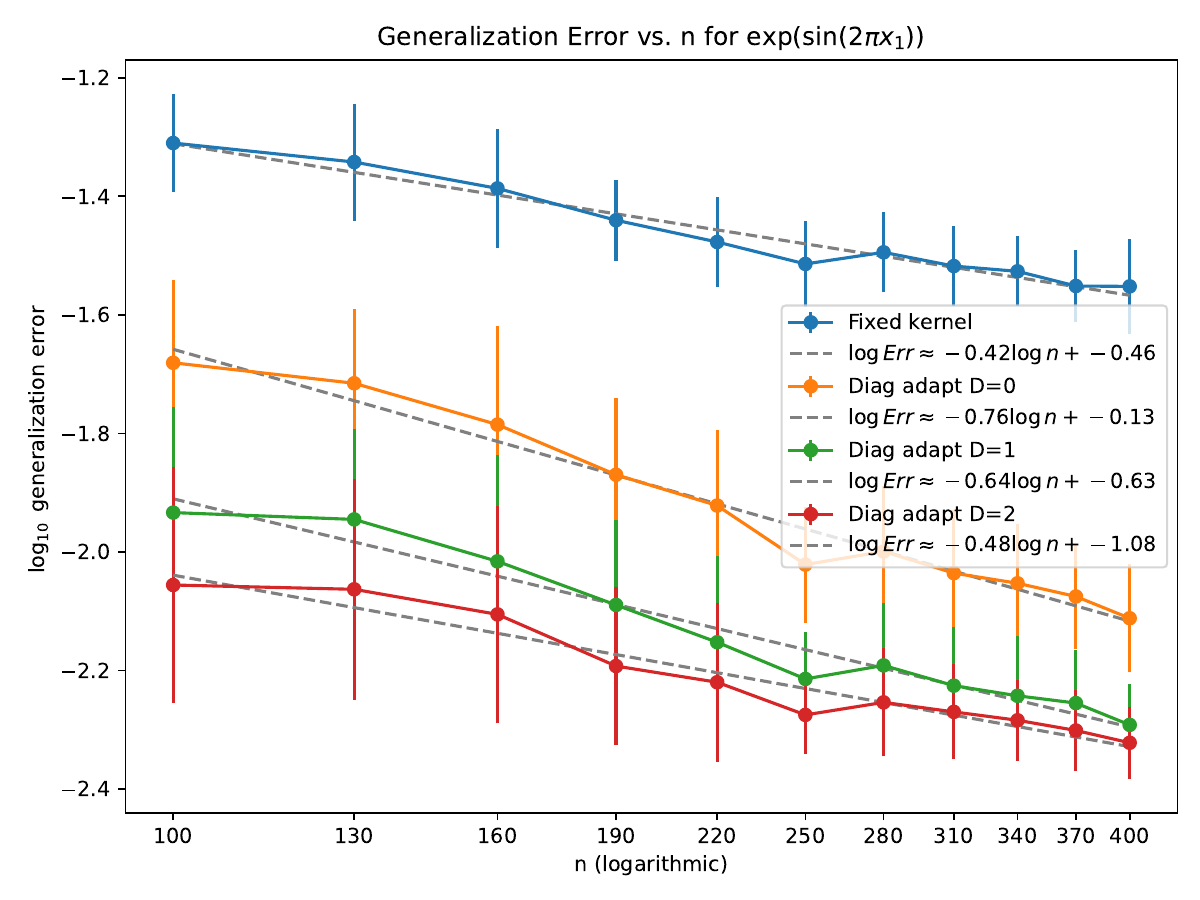}
    \end{minipage}%
  }%
  \subfigure{
    \begin{minipage}[t]{0.33\linewidth}
      \centering
      \includegraphics[width=1.\linewidth]{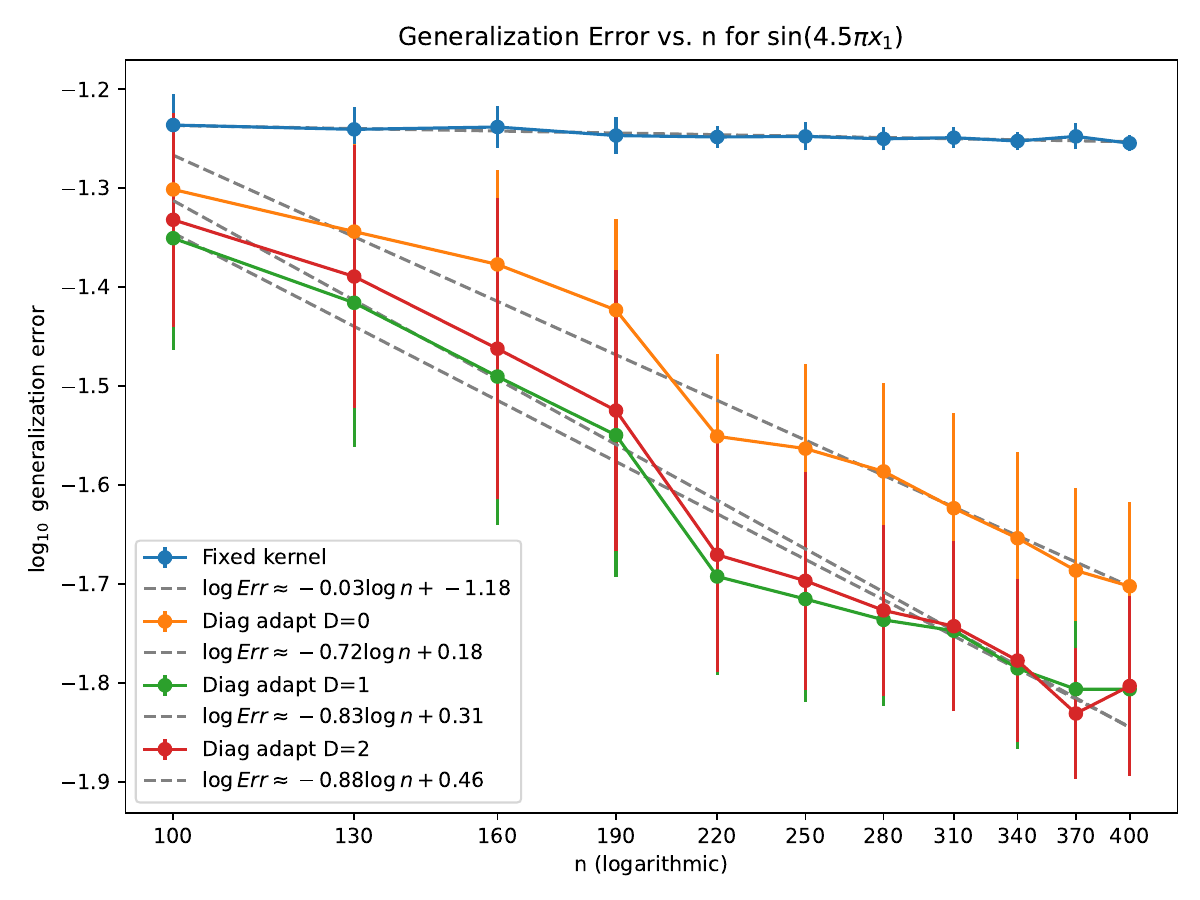}
    \end{minipage}%
  }%

  \subfigure{
    \begin{minipage}[t]{0.33\linewidth}
      \centering
      \includegraphics[width=1.\linewidth]{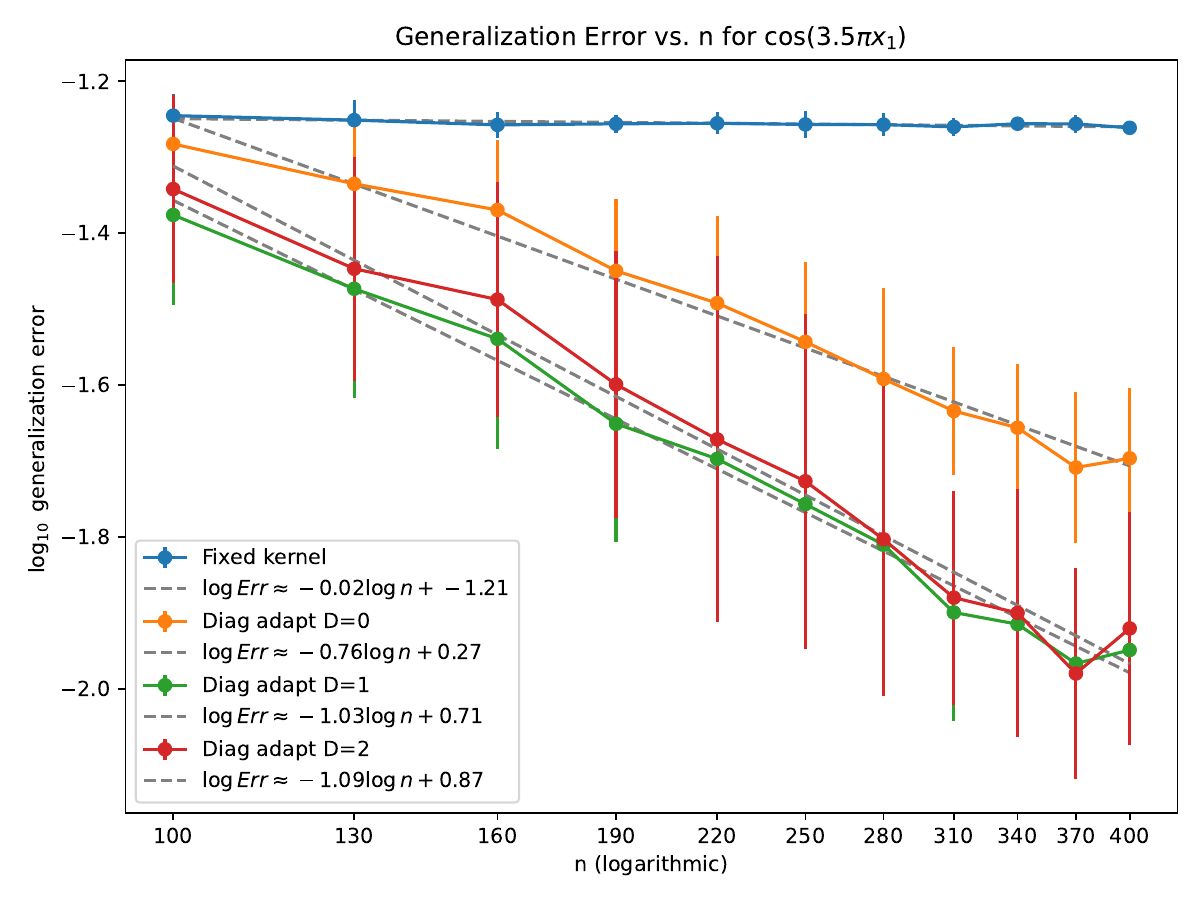}
    \end{minipage}%
  }%
  \subfigure{
    \begin{minipage}[t]{0.33\linewidth}
      \centering
      \includegraphics[width=1.\linewidth]{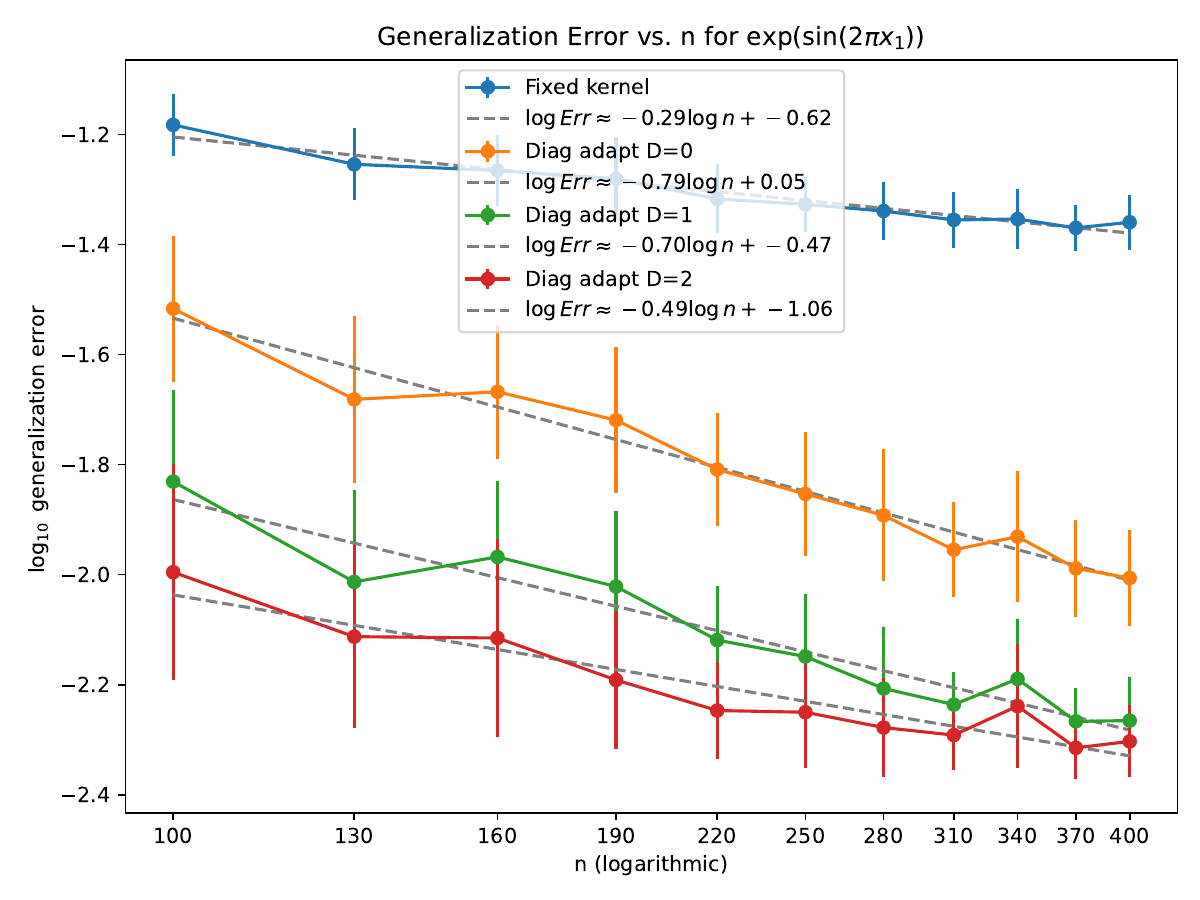}
    \end{minipage}%
  }%
  \subfigure{
    \begin{minipage}[t]{0.33\linewidth}
      \centering
      \includegraphics[width=1.\linewidth]{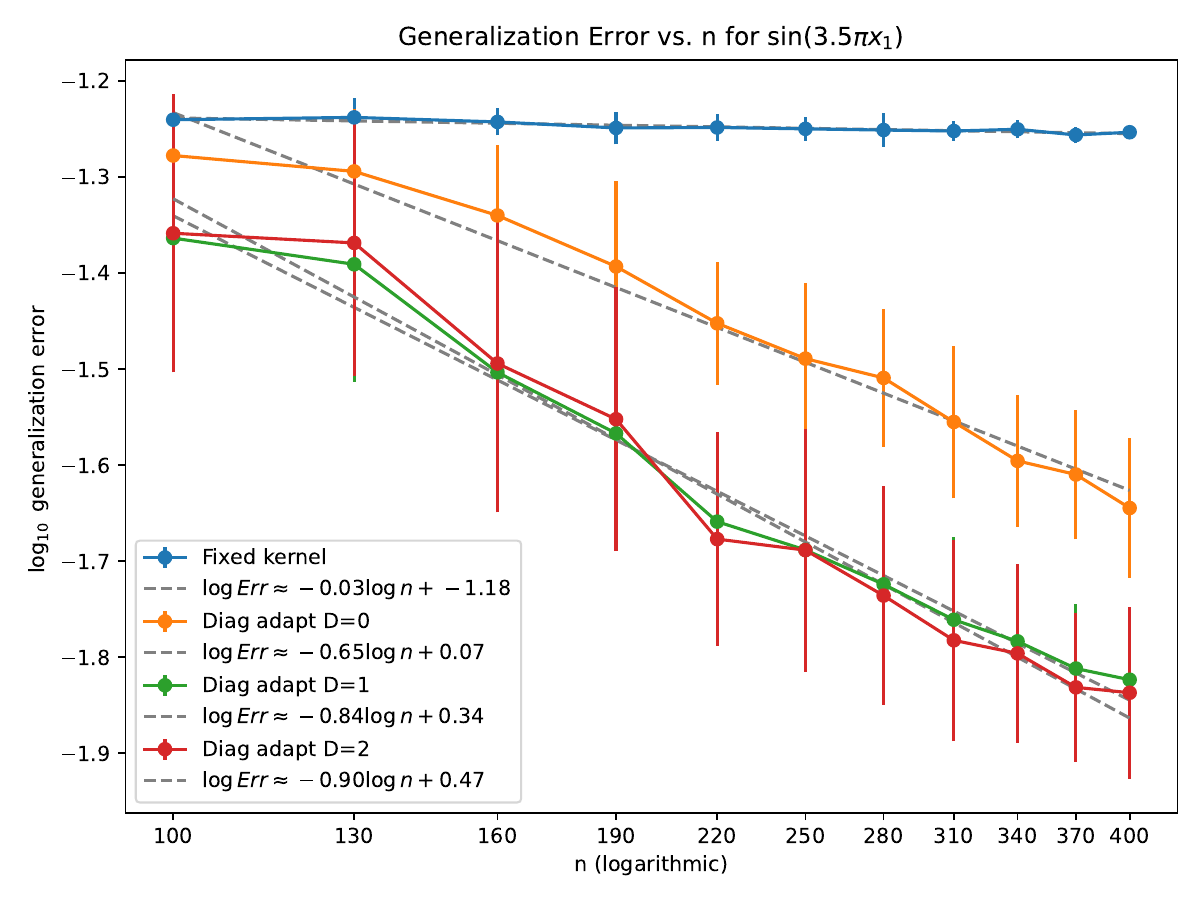}
    \end{minipage}%
  }%

  \centering
  \caption{
  Comparison of the fixed kernel method and the diagonal adaptive kernel method in various settings.
  The three rows correspond to dimension $d=2,3,4$ in \cref{example:low-dim-structure} respectively.
  The error bars represent the standard deviation over 32 runs.
  }

  \label{fig:ComparisonRate}
\end{figure}

%
%


\acks{Qian Lin's research was supported in part by the National Natural Science Foundation of China (Grant 92370122, Grant 11971257).}



\appendix


\section{Proof for Multilayer Model}
\label{sec:proofmulti}

In this section, we provide the proof for the multilayer model,
which resembles the proof for the two-layer model in \cref{sec:proofs} but is more complicated.

Let us also denote $E(x) = \xk{e_j(x)}_{j\geq 1}$ as a column vector.
In this section, we denote
\begin{align}
  \bm{\theta}(t) = \xk{\theta_j}_{j \geq 1} = \bm{a}(t) \odot \bm{b}^{D}(t),
\end{align}
where we use $\bm{b}^{D}$ to denote the element-wise power of $\bm{b}$.
Then, similar to the single-layer model, we have
\begin{align*}
  f(t) = \ang{\Phi_{\bm{a}(t),\bm{b}(t)}(x), \bm{\beta}(t)} =
  \ang{E(x), \bm{a}(t)\odot \bm{b}^{\odot D} \odot \bm{\beta}(t)} = \ang{E(x), \bm{\theta}(t)} = E(x)^{\T} \bm{\theta}(t),
\end{align*}
Then, following \cref{subsec:Proof_ExplicitForm}, we can derive the gradient flow of the multilayer model:
\begin{align}
  \label{eq:MultiLayer_GradientFlow}
  \left\{
  \begin{aligned}
    \dot{\bm{a}} &= \bm{b}^D \odot \bm{\beta} \odot \bm{\Delta},\quad a_j(0) = \lambda_j^{\hf} \\
    \dot{\bm{b}} &= D \bm{a}\odot \bm{b}^{D-1}\odot \bm{\beta} \odot \bm{\Delta},\quad b_j(0) = b_0\\
    \dot{\beta} &= \bm{a}\odot \bm{b}^D \odot \bm{\Delta},\quad \bm{\beta}(0) = \bm{0},
  \end{aligned}
  \right.
\end{align}
where
\begin{align*}
  \Delta \coloneqq \hat{\Sigma}(\bm{\theta}^* - \bm{\theta}) + \bm{r},\quad
  \hat{\Sigma} = \frac{1}{n} \sum_{i=1}^n E(x_i) E(x_i)^\T, \quad
  \bm{r} = \frac{1}{n} \sum_{i=1}^n \ep_i E(x_i).
\end{align*}

\subsection{Equation Analysis for Multilayer Model}
Let us consider the perturbed one-dimensional equation for the multilayer model.
Let $h(t): \R_{\geq 0} \to \R$ be a continuous perturbation function.
Let us consider now the following perturbed one-dimensional dynamics:
\begin{align}
  \label{eq:MultiLayer1d_Perturbed}
  \left\{
  \begin{aligned}
    \dot{a} &= b^D \beta \xk{z - \theta +h},\quad a(0) = \lambda^{1/2} > 0,\\
    \dot{b} &= D a b^{D-1} \beta \xk{z - \theta +h},\quad b(0) = b_0,\\
    \dot{\beta} &= a b^D \xk{z - \theta +h}, \quad \beta(0) = 0,
  \end{aligned}
  \right.
\end{align}
where $\theta(t) = a(t) b^D(t) \beta(t)$.

\noindent\textit{Conservation quantities.}
It is easy to see that
\begin{align*}
  \dv{t} a^2 = \frac{1}{D} \dv{t} b^2 = \dv{t} \beta^2 = 2 a b^D \beta \xk{z - \theta + h}.
\end{align*}
Consequently, we have
\begin{align}
  \label{eq:MultiLayer_Conservation}
  a^2(t) - \beta^2(t) \equiv \lambda,\qquad
  b^2(t) - D \beta^2(t) \equiv b_0^2.
\end{align}
Using this, we see that
\begin{align}
  \label{eq:MultiLayer_PositivenessAB}
  a(t) = \xk{\beta^2(t) + \lambda}^{1/2} > 0,\qquad b(t) = \xk{D \beta^2(t) + b_0^2}^{1/2} > 0.
\end{align}
Using these conservation quantities, we can prove the following estimations in terms of $\beta$:
\begin{align}
  \label{eq:MultiLayerEq_Estimations}
  \begin{aligned}
    & \max(\lambda^{\hf},\abs{\beta}) \leq a \leq \sqrt {2} \max(\lambda^{\hf},\abs{\beta}),\\
    & \max(b_0,\sqrt{D} \abs{\beta}) \leq b \leq \sqrt {2} \max(b_0,\sqrt{D} \abs{\beta}),
  \end{aligned}
\end{align}
which also implies that $\abs{\theta} = \abs{ab^D\beta} \geq D^{D/2} \abs{\beta}^{D+2}$.

\noindent\textit{The evolution of $\theta$.}
It is direct to compute that
\begin{align}
  \label{eq:MultiLayerEq_Theta}
  \begin{aligned}
    \dot{\theta} &= \dot{a} b^D \beta + a D b^{D-1} \dot{b} \beta + a b^D \dot{\beta} \\
    &= \zk{(b^D \beta)^2 + (D a b^{D-1} \beta)^2 + (a b^D)^2} (z - \theta + h) \\
    &= \theta^2 (a^{-2} + D b^{-2} + \beta^{-2}) (z - \theta + h).
  \end{aligned}
\end{align}
The following estimation of the equation will be frequently used in the proof.
Using \cref{eq:MultiLayerEq_Estimations}, we have
\begin{align*}
  \abs{\theta} = \abs{ab^D\beta} \geq D^{D/2} \abs{\beta}^{D+2}, \quad \Longrightarrow \quad
  \abs{\beta} \leq \xk{D^{-D/2} \abs{\theta}}^{1/(D+2)}.
\end{align*}
Therefore,
\begin{align}
  \label{eq:MultiLayerEq_ThetaEstimation}
  \theta^2 (a^{-2} + D b^{-2} + \beta^{-2}) \geq \theta^2 \beta^{-2}
  \geq D^{-\frac{D}{D+2}} \abs{\theta}^{\frac{2D+2}{D+2}}.
\end{align}

\begin{lemma}[Perturbation bound]
  Consider the equation \cref{eq:MultiLayer1d_Perturbed}.
  Let $t_0 \geq 0$.
  Suppose there exists some $\kappa > 0$ such that $\abs{h(t)} \leq \kappa$ for all $t \geq t_0$.
  Then, for $t \geq t_0$,
  \begin{enumerate}[(i)]
    \item $\abs{z - \theta(t)}$ is decreasing provided that $\abs{z - \theta(t)} \geq \kappa$.
    \item Once $\abs{z - \theta(t_1)} \leq \kappa$ for some $t_1 \geq t_0$,
    we have $\abs{z - \theta(t)} \leq \kappa$ for all $t \geq t_1$.
  \end{enumerate}
\end{lemma}
\begin{proof}
  The same as \cref{lem:TwoLayer_Monotonicity}.
\end{proof}

\begin{lemma}[Approaching from below, large $\kappa$]
  \label{lem:MultiLayer_AppBelowLargeKappa}
  Consider the equation \cref{eq:MultiLayer1d_Perturbed} with $z \geq 0$ (a similar result holds for $z \leq 0$).
  Let $t_0 \geq 0$.
  Suppose that there exists some $\kappa > 0$ such that $\abs{h(t)} \leq \kappa,~ \forall t \geq t_0$ and
  there exists some $M\geq 0$ such that
  \begin{align*}
    z - M \leq \theta(t_0) \leq z.
  \end{align*}
  Then, we have
  \begin{align*}
    \abs{z - \theta(t)} \leq 2\kappa,\quad \forall t \geq t_0 + \overline{T}^{\mr{app}},
  \end{align*}
  where
  \begin{align}
    \overline{T}^{\mr{app}} \leq 4\zk{D^{D/2}\kappa R^D}^{-1} \ln \frac{R}{r},

  \end{align}
  and $r = \min(\lambda^{1/2},b_0/\sqrt {D})$, $R = \max(\lambda^{1/2},b_0/\sqrt {D})$.
\end{lemma}
\begin{proof}
  The proof is similar to \cref{lem:TwoLayer_AppBelowLargeKappa} but is more complicated due to the $b$ term.
  Denoting
  \begin{align*}
    r = \min(\lambda^{1/2},b_0/\sqrt {D}),\quad R = \max(\lambda^{1/2},b_0/\sqrt {D}),
  \end{align*}
  let us further define
  \begin{align*}
    T^{\mr{neg},2} &= \inf \dk{ s\geq 0 : \beta(t_0 + s) \geq -R},\quad
    T^{\mr{neg},1} = \inf \dk{ s\geq 0 : \beta(t_0 + s) \geq -r},\\
    T^{\mr{esc},1} &= \inf \dk{ s\geq 0 : \beta(t_0+s) \geq r},\quad
    T^{\mr{esc},2} = \inf \dk{ s\geq 0 : \beta(t_0+s) \geq R},\\
    T^{\mr{app}} &= \inf \dk{ s\geq 0 : \abs{z -\theta(t_0+s)} \leq 2\kappa},
  \end{align*}
  and it suffices to bound $T^{\mr{app}}$ supposing that
  \begin{align*}
    0 < T^{\mr{neg},2} < T^{\mr{neg},1} < T^{\mr{esc},1} < T^{\mr{esc},2} < T^{\mr{app}}.
  \end{align*}
  We also recall that
  \begin{align}
    \label{eq:Proof_MultiLayer_AppBound}
    \dot{\beta}(t) = a(t) b^D(t) \zk{z - \theta(t) + h(t)} \geq \kappa a(t)b^D(t) ,\qq{for} t \leq T^{\mr{app}}.
  \end{align}

  \noindent \textit{Stage $0 \leq s \leq T^{\mr{neg},2}$}\quad
  Before $T^{\mr{neg},2}$, we have $\beta(t) \leq -R$ and thus \cref{eq:MultiLayerEq_Estimations} gives
  \begin{align*}
    z -M \leq \theta(t_0) \leq - D^{\frac{D}{2}} 2^{\frac{D+1}{2}} \abs{\beta(t_0)}^{D+2},\quad \Longrightarrow
    \beta(t_0) \geq - \xk{D^{-\frac{D}{2}} 2^{-\frac{D+1}{2}} M}^{\frac{1}{D+2}}
  \end{align*}
  Moreover, we combine \cref{eq:Proof_MultiLayer_AppBound} and \cref{eq:MultiLayerEq_Estimations} that $a \geq \abs{\beta}$ and $b \geq \sqrt {D}\abs{\beta}$ to get
  \begin{align*}
    \dot{\beta}(t_0 +s) \geq  D^{\frac{D}{2}} \kappa \abs{\beta(t_0 +s)}^{D+1}, \qq{for} s \in [0,T^{\mr{neg,2}}],
  \end{align*}
  so
  \begin{align*}
    \abs{\dot{\beta}(t_0 +s)} \leq \zk{\abs{\beta(t_0)}^{-D} + D D^{\frac{D}{2}} \kappa s}^{-\frac{1}{D}}, \qq{for} s \in [0,T^{\mr{neg,2}}],
  \end{align*}
  which implies
  \begin{align*}
    T^{\mr{neg},2} \leq \xk{D^{\frac{D+2}{2}} R^D \kappa}^{-1}.
  \end{align*}

  \noindent \textit{Stage $T^{\mr{neg},2} \leq s \leq T^{\mr{neg},1}$}\quad
  After $T^{\mr{neg},2}$, we consider the two cases of $\lambda^{\hf} \leq b_0/\sqrt {D}$ and $\lambda^{\hf} \geq b_0/\sqrt {D}$.
  We can use the alternative lower bounds in \cref{eq:MultiLayerEq_Estimations} to get
  \begin{align*}
    \dot{\beta}(t_0 +s) \geq \kappa b_0^D \abs{\beta(t_0 +s)},
  \end{align*}
  or
  \begin{align*}
    \dot{\beta}(t_0 +s) \geq  \kappa \lambda^{\hf} D^{\frac{D}{2}} \abs{\beta(t_0 +s)}^D.
  \end{align*}

  Therefore, for the first case, we have
  \begin{align*}
    T^{\mr{neg},1} - T^{\mr{neg},2} \leq \xk{\kappa b_0^D}^{-1} \ln \frac{R}{r}.
  \end{align*}
  For the second case, we have
  \begin{align*}
    T^{\mr{neg},1} - T^{\mr{neg},2} \leq
    \begin{cases}
      \xk{\kappa \lambda^{\hf}}^{-1} \ln \frac{R}{r},\qq{if} D = 1,\\
      \zk{(D-1)\kappa D^{\frac{D}{2}} \lambda^{\frac{D}{2}} }^{-1},\qq{if} D \geq 2.
    \end{cases}
  \end{align*}
  Combining these two cases, we have
  \begin{align*}
    T^{\mr{neg},1} - T^{\mr{neg},2} \leq \xk{\kappa (D^{\hf} R)^D}^{-1} \ln \frac{R}{r}.
  \end{align*}

  \noindent \textit{Stage $T^{\mr{neg},1} \leq s \leq T^{\mr{esc},1}$}\quad
  Let us now consider the bound
  \begin{align*}
    \dot{\beta}(t_0 +s) \geq \kappa \lambda^{\hf} b_0^D,
  \end{align*}
  so that
  \begin{align*}
    T^{\mr{esc},1} - T^{\mr{neg},1} \leq 2 r \xk{\kappa \lambda^{\hf} b_0^D}^{-1}.
  \end{align*}

  \noindent \textit{Stage $T^{\mr{esc},1} \leq s \leq T^{\mr{esc},2}$}\quad
  This stage is very similar to the stage of $T^{\mr{neg},2} \leq s \leq T^{\mr{neg},1}$ and we have
  \begin{align*}
    T^{\mr{neg},1} - T^{\mr{neg},2} \leq \xk{\kappa (D^{\hf} R)^D}^{-1} \ln \frac{R}{r}.
  \end{align*}

  \noindent \textit{Stage $T^{\mr{esc},2} \leq s \leq T^{\mr{app}}$}\quad
  This stage resembles the stage $0 \leq s \leq T^{\mr{neg},2}$ and we can obtain that
  \begin{align*}
    \beta(t_0 + T^{\mr{neg},2} + s_1) \geq \zk{\beta(t_0 + T^{\mr{neg},2}) - D^{\frac{D+2}{2}} \kappa s_1}^{-1/D},
  \end{align*}
  and thus
  \begin{align*}
    T^{\mr{app}} - T^{\mr{esc},2} \leq \xk{D^{\frac{D+2}{2}} R^D \kappa}^{-1}.
  \end{align*}

  Finally, we notice that the dominating term is $\xk{\kappa (D^{\hf} R)^D}^{-1} \ln (R/r)$.

\end{proof}

\begin{lemma}[Signal time from below, small $\kappa$]
  \label{lem:MultiLayer_AppBelowSmallKappa}
  Consider the same setting as in \cref{lem:MultiLayer_AppBelowSmallKappa}.
  Assume further that $z \geq 2\kappa$.
  Then,
  \begin{align*}
    \theta(t) \geq z/4,\quad \forall t \geq t_0 + \overline{T}^{\mr{sig}},
  \end{align*}
  where
  \begin{align*}
    \overline{T}^{\mr{sig}} = 4  \zk{D^{D/2} z R^D}^{-1} \ln \frac{R}{r},
  \end{align*}
  and $r = \min(\lambda^{1/2},b_0/\sqrt {D})$, $R = \max(\lambda^{1/2},b_0/\sqrt {D})$.
\end{lemma}
\begin{proof}
  The proof resembles that of \cref{lem:MultiLayer_AppBelowLargeKappa}, but we use a different bound for the term:
  \begin{align*}
    z - \theta(t) + h(t) \geq \frac{3}{4}z - \kappa \geq \frac{1}{4}z,\qq{for} t \in [t_0, T^{\mr{sig}}],
  \end{align*}
  where
  \begin{align*}
    T^{\mr{sig}} = \inf \dk{ s\geq 0 : \theta(t_0+s) \geq z/4},
  \end{align*}
  so the results follow by replacing $\kappa$ with $z/4$.
\end{proof}

\begin{lemma}[Approaching from above]
  \label{lem:MultiLayer_AppAbove}
  Consider the equation \cref{eq:MultiLayer1d_Perturbed} with $z \geq 0$ (a similar result holds for $z \leq 0$).
  Let $t_0 \geq 0$.
  Suppose that there exists some $\kappa > 0$ such that $\abs{h(t)} \leq \kappa,~ \forall t \geq t_0$ and $\theta(t_0) \geq z$.
  Then, denoting $m = 2\max(z,2\kappa)$, we have
  \begin{align*}
    \abs{z - \theta(t)} \leq m,\quad \forall t \geq t_0 + \frac{D+2}{D+1} D^{\frac{D}{D+2}} m^{-\frac{2D+2}{D+2}}.
  \end{align*}
\end{lemma}
\begin{proof}
  The proof is similar to \cref{lem:TwoLayer_AppAbove}.
  Defining
  \begin{align*}
    T^{\mr{sig}} = \inf \dk{ s \geq 0 : \theta(t_0 + s) \leq z+2\max(z,2\kappa)},
  \end{align*}
  we have
  \begin{align*}
    \frac{1}{2}\theta(t_0 + s) \geq \frac{1}{2} z + \max(z,2\kappa) \geq z + \kappa,\qq{for} s \in [0,T^{\mr{sig}}],
  \end{align*}
  and thus $\theta - z - \kappa \geq \frac{1}{2}\theta$.
  Plugging this and also \cref{eq:MultiLayerEq_ThetaEstimation} into \cref{eq:MultiLayerEq_Theta}, we get
  \begin{align*}
    \dot{\theta} &= \theta^2 (a^{-2} + D b^{-2} + \beta^{-2}) (z - \theta +h) \\
    &\leq - D^{-\frac{D}{D+2}} \theta^{\frac{2D+2}{D+2}} \cdot \frac{1}{2}\theta \\
    &= - \frac{1}{2} D^{-\frac{D}{D+2}} \theta^{1 + \frac{2D+2}{D+2}},
  \end{align*}
  which gives the result by \cref{prop:ODE_Power}.

\end{proof}

\begin{lemma}[Approximation time near $z$]
  \label{lem:MultiLayer_FinalTime}
  Consider the equation \cref{eq:MultiLayer1d_Perturbed} with $z \geq 0$ (a similar result holds for $z \leq 0$).
  Let $t_0 \geq 0$.
  Suppose that there exists some $\kappa > 0$ such that $\abs{h(t)} \leq \kappa,~ \forall t \geq t_0$.
  Suppose also that
  \begin{align*}
    \frac{1}{4}z \leq \theta(t_0) \leq 3 z.
  \end{align*}
  Then, for any $\delta > 0$, we have
  \begin{align*}
    \abs{z - \theta(t)} \leq \kappa + \delta,\quad \forall
    t \geq t_0 + 4^{\frac{2D+2}{D+2}}   D^{\frac{D}{D+2}} z^{-\frac{2D+2}{D+2}} \ln^+ \frac{\abs{z - \theta(t_0)} - \kappa}{\delta}.
  \end{align*}
\end{lemma}
\begin{proof}
  The proof is similar to \cref{lem:TwoLayer_FinalTime}:
  we define
  \begin{align*}
    T^{\mr{app}} = \inf \dk{ s\geq 0 : \abs{z- \theta(t_0+s)} \leq \kappa+\delta},
  \end{align*}
  and provide an upper bound of $T^{\mr{app}}$ when it is not zero.
  By monotonicity,  we have $\frac{1}{4}z \leq \theta(t) \leq 3 z$ for all $t \geq t_0$.

  Now, we recall \cref{eq:MultiLayerEq_Theta} and \cref{eq:MultiLayerEq_ThetaEstimation} that
  \begin{align*}
    \dot{\theta} = \theta^2 (a^{-2} + D b^{-2} + \beta^{-2}) (z - \theta + h),
  \end{align*}
  and
  \begin{align*}
    \theta^2 (a^{-2} + D b^{-2} + \beta^{-2}) \geq D^{-\frac{D}{D+2}} \theta^{\frac{2D+2}{D+2}}
    \geq 4^{-\frac{2D+2}{D+2}}   D^{-\frac{D}{D+2}} z^{\frac{2D+2}{D+2}} \eqqcolon k,
  \end{align*}
  where we used the fact that $\theta \geq z/4$.

  Plugging it into the equation of $\theta(t)$ and considering two cases $\theta(t_0) \leq z - \kappa$ and $\theta(t_0) \geq z + \kappa$,
  we have
  \begin{align*}
    z - \kappa - \theta(t_0 + s) &\leq (z - \kappa - \theta(t_0)) \exp(-k s),\qq{if} \theta(t_0) \leq z - \kappa, \\
    \theta(t_0 + s) - z - \kappa &\leq (\theta(t_0) - z - \kappa) \exp(-k s),\qq{if} \theta(t_0) \geq z + \kappa,
  \end{align*}
  and thus
  \begin{align*}
    T^{\mr{app}} \leq k^{-1} \ln \frac{\abs{z - \theta(t_0)} - \kappa}{\delta}.
  \end{align*}
\end{proof}

\begin{corollary}
  \label{cor:MultiLayer_AppAboveLargeKappa}
  Under the same setting as in \cref{lem:MultiLayer_AppAbove}, we have
  \begin{align*}
    \abs{z - \theta(t)} \leq 4 \kappa,\quad \forall
    t \geq t_0 + C_D \kappa^{-\frac{2D+2}{D+2}}\xk{1 + \ln^+ \frac{\abs{z - \theta(t_0)}}{\kappa}},
  \end{align*}
  where $C_D$ is some constant depending on $D$.
\end{corollary}
\begin{proof}
  Similar to the proof of \cref{cor:TwoLayer_AppAboveLargeKappa} but we use \cref{lem:MultiLayer_AppAbove} and \cref{lem:MultiLayer_FinalTime}.

\end{proof}

\begin{corollary}
  \label{cor:MultiLayer_AppBoundSmallKappa}
  Consider the equation \cref{eq:MultiLayer1d_Perturbed} with $z \geq 0$ (a similar result holds for $z \leq 0$).
  Let $t_0 \geq 0$.
  Suppose that there exists some $\kappa > 0$ such that $\abs{h(t)} \leq \kappa,~ \forall t \geq t_0$ and
  $z \geq 2\kappa$.
  Then,
  \begin{align*}
    \abs{z - \theta(t)} \leq \kappa + \delta,\quad \forall t \geq t_0 + \overline{T}^{\mr{app}}(\delta),
  \end{align*}
  where
  \begin{align*}
    \overline{T}^{\mr{app}}(\delta) \leq C_{D,1} z^{-1} R^{-D} \ln \frac{R}{r} + C_{D,2} z^{-\frac{2D+2}{D+2}} \ln^+ \frac{2z}{\delta},
  \end{align*}
  $r = \min(\lambda^{1/2},b_0/\sqrt {D})$, $R = \max(\lambda^{1/2},b_0/\sqrt {D})$ and $C_{D,1}, C_{D,2}$ are two constants depending only on $D$.
\end{corollary}
\begin{proof}
  It is a consequence of \cref{lem:MultiLayer_AppBelowSmallKappa}, \cref{lem:MultiLayer_AppAbove} and \cref{lem:MultiLayer_FinalTime}.
\end{proof}

\begin{lemma}[Retracting from the negative]
  \label{lem:MultiLayer_RectractNeg}
  Consider the equation \cref{eq:MultiLayer1d_Perturbed} with $z \geq 0$ (a similar result holds for $z \leq 0$).
  Let $t_0 \geq 0$.
  Suppose that there exists some $\kappa > 0$ such that $\abs{h(t)} \leq \kappa,~ \forall t \geq t_0$.
  Then, we have
  \begin{align*}
    \theta(t) \geq - 2 \kappa,\quad \forall t \geq t_0 +\frac{D+2}{D+1}D^{\frac{D}{D+2}} (2\kappa)^{-\frac{2D+2}{D+2}}.
  \end{align*}
\end{lemma}
\begin{proof}
  The proof is done by modifying the proof of \cref{lem:TwoLayer_RectractNeg} with the proof of \cref{lem:MultiLayer_AppAbove}.

\end{proof}

\begin{corollary}[Retracting into $\kappa$]
  \label{cor:MultiLayer_Rectracting}
  Consider the equation \cref{eq:MultiLayer1d_Perturbed} .
  Let $t_0 \geq 0$.
  Suppose that there exists some $\kappa > 0$ such that $\abs{h(t)} \leq \kappa,~ \forall t \geq t_0$.
  Then,
  \begin{align*}
    \abs{z - \theta(t)} \leq 2\max(\abs{z},2\kappa),\quad \forall t \geq t_0 + \frac{D+2}{D+1}D^{\frac{D}{D+2}} (2\kappa)^{-\frac{2D+2}{D+2}}.
  \end{align*}
\end{corollary}
\begin{proof}
  If $\theta(t_0) \geq z$, we use \cref{lem:MultiLayer_AppAbove}.
  If $\theta(t_0) \leq 0$, we use \cref{lem:MultiLayer_RectractNeg} and note that $\abs{z - \theta(t)} \leq \abs{z}+\abs{\theta(t)}$.
\end{proof}

\begin{lemma}[Error control]
  \label{lem:MultiLayer_ErrorControl}
  Consider the equation \cref{eq:MultiLayer1d_Perturbed}.
  As long as
  \begin{align*}
    2^{\frac{D+1}{2}} \lambda^{\hf}b_0^D \int_0^t \xk{z + \abs{h(s)}} \dd s \leq \min(\lambda^{1/2},b_0/\sqrt {D}),
  \end{align*}
  we have
  \begin{align*}
    \abs{\theta(t)} \leq 2^{\frac{D+1}{2}} \lambda b_0^{D} \int_0^t \xk{\abs{z} + \abs{h(s)}} \dd s.
  \end{align*}
  Moreover, if $\lambda^{\hf} \leq b_0/\sqrt {D}$ and
  \begin{align}
    \label{eq:MultiLayer_ErrorControl_Condition}
    2^{\frac{D+1}{2}} b_0^D \int_{0}^{t} \xk{\abs{z} + \abs{h(s)}} \dd s \leq \ln \frac{b_0/\sqrt {D}}{\lambda^{1/2}},
  \end{align}
  We have
  \begin{align}
  \label{eq:MultiLayer_ErrorControl_Beta}
  \abs{\beta(t)} &\leq \lambda^{\hf} \exp(2^{\frac{D+1}{2}} b_0^D \int_{0}^t \xk{\abs{z} + \abs{h(\tau)}} \dd \tau) \\
    \label{eq:MultiLayer_ErrorControl}
    \abs{\theta(t)} &\leq 2^{\frac{D+1}{2}}  \lambda b_0^{D} \exp(2^{\frac{D+3}{2}} b_0^D \int_{0}^{t} \xk{\abs{z} + \abs{h(s)}} \dd s).
  \end{align}
\end{lemma}
\begin{proof}
  Without loss of generality, we assume that $z \geq 0$.
  Similar to the argument in the proof of \cref{lem:TwoLayer_ErrorControl}, from \cref{eq:MultiLayer1d_Perturbed}, we have
  \begin{align}
    \abs{\beta(t)} \leq \int_0^t a(s) b^D(s) \xk{z + \abs{h(s)}} \dd s.
  \end{align}
  Denoting
  \begin{align*}
    r = \min(\lambda^{1/2},b_0/\sqrt {D}),\quad R = \max(\lambda^{1/2},b_0/\sqrt {D}),
  \end{align*}
  let us define
  \begin{align*}
    T^{(1)} = \inf \dk{ t\geq 0 : \abs{\beta(t)} \geq r},\quad
    T^{(2)} = \inf \dk{ t\geq 0 : \abs{\beta(t)} \geq R}.
  \end{align*}

  When $t \leq T^{(1)}$, we use \cref{eq:MultiLayerEq_Estimations} to get
  $a(t) \leq \sqrt {2}\lambda^{\hf}$ and $b(t) \leq \sqrt {2} b_0$, so
  \begin{align*}
    \abs{\beta(t)} \leq 2^{\frac{D+1}{2}}  \lambda^{\hf}  b_0^D \int_0^t \xk{z + \abs{h(s)}} \dd s,
  \end{align*}
  which shows that
  \begin{align*}
    T^{(1)} \geq \inf\dk{t \geq 0 : 2^{\frac{D+1}{2}} \lambda^{\hf}b_0^D \int_0^t \xk{z + \abs{h(s)}} \dd s \geq r}.
  \end{align*}

  When $t > T^{(1)}$, we consider the case $\lambda^{\hf} \leq b_0/\sqrt {D}$.
  Then we have $a(t) \leq \sqrt {2} \abs{\beta(t)}$ and $b(t) \leq \sqrt {2} b_0$, so
  \begin{align*}
    \abs{\beta(t)} \leq \lambda^{\hf} + 2^{\frac{D+1}{2}} b_0^D \int_{T^{(1)}}^t \abs{\beta(s)} \xk{z + \abs{h(s)}} \dd s,
  \end{align*}
  and thus
  \begin{align*}
    \abs{\beta(T^{(1)} + s)} \leq \lambda^{\hf} \exp(2^{\frac{D+1}{2}} b_0^D \int_{T^{(1)}}^{T^{(1)}+s} \xk{z + \abs{h(\tau)}} \dd \tau),\quad s \leq T^{(2)} - T^{(1)},
  \end{align*}
  implying that
  \begin{align*}
    T^{(2)} - T^{(1)} \geq \inf \dk{t \geq T^{(1)} : 2^{\frac{D+1}{2}} b_0^D \int_{T^{(1)}}^{T^{(1)}+s} \xk{z + \abs{h(\tau)}} \dd \tau = \ln \frac{b_0/\sqrt {D}}{\lambda^{1/2}}}.
  \end{align*}
  The bound for $\theta(t)$ now follows from using the bounds $a \leq \sqrt{2}\abs{\beta}$, $b \leq \sqrt{2} b_0$ to get
  \begin{align*}
      \abs{\theta(T^{(1)} + s)} = \abs{a b^D \beta} 
      \leq 2^{\frac{D+1}{2}}  b_0^D \abs{\beta}^{2}
      \leq 2^{\frac{D+1}{2}}  b_0^D \lambda \exp(2^{\frac{D+3}{2}} b_0^D \int_{T^{(1)}}^{T^{(1)}+s} \xk{z + \abs{h(\tau)}} \dd \tau).
  \end{align*}
  Finally, the upper bounds in the statement of the lemma are rougher and also apply to $t \leq T^{(1)}$.
\end{proof}

\subsection{Proof of \cref{thm:EigenvalueDeepGD}}

In the following, we will choose
\begin{align*}
  c_{\mr{b}} n^{-\frac{D}{2(D+2)}} \leq b_0^D \leq C_{\mr{b}} n^{-\frac{D}{2(D+2)}},
\end{align*}
where $c_{\mr{b}}, C_{\mr{b}} > 0$ are some constants to be determined and will be tracked.

For some small $s' > 0$, we choose the signal components $S$ by
\begin{align}
  S = S_1 \cup S_2 = \dk{j \geq 1 : \abs{\theta^*_j} \geq n^{-1/2}\sqrt {\ln n}} \cup
  \dk{j\geq 1: \lambda_j \geq n^{-\frac{1+s'}{D+2}}}.
\end{align}
Then, similar to \cref{eq:Proof_Component_S}, we also have
\begin{align}
  \label{eq:Proof_Component_S2}
  \abs{S} \leq C n^{(1-s_0)/2},\quad \max S \leq C n^{\kappa/2}.
\end{align}
For the noise components $R = S^{\complement}$,
we have
\begin{align}
  \label{eq:Proof_Component_L2}
  L \coloneqq \min R \geq c n^{\frac{1+s'}{(D+2)\gamma}}.
\end{align}

Similar to the proof in \cref{subsec:Proof_MainThm}, we introduce the error time
\begin{align}
  \label{eq:ErrorTimeMulti}
  T^{\mr{err}} = \inf \dk{t \geq 0 : \abs{\theta_k(t)} \leq 2^{\frac{D+2}{2}} \lambda_k b_0^D \exp(E \sqrt {\ln n + \ln k}),\quad \forall k \in R},
\end{align}
where $E$ is also a constant to be determined.

Now, we can use a similar argument as in \cref{subsec:Proof_MainThm} to show \cref{eq:Proof_BasicControls}, which we repeat here for readability:
\begin{align}
  \label{eq:Proof_BasicControls2}
  \begin{aligned}
    \abs{p_k} \leq \eta_k \norm{(\bm\theta^* - \bm\theta)_S}_{\infty}, \quad
    \abs{q_k}  \leq C n^{-1/2}\sqrt{\ln (nk)}, \quad
    \abs{r_k} \leq C n^{-1/2}\sqrt{\ln (nk)},
  \end{aligned}
\end{align}
where $\eta_k = C \abs{S} n^{-1/2}\sqrt{\ln (nk)} \leq C n^{-s} \sqrt{\ln (nk)}$.
%

\subsubsection{Signal Shrinkage}

Following \cref{subsubsec:Proof_Shrinkage},
we still have \cref{eq:Proof_BasicControls_S}.
Now, we provide the following proposition as a multilayer extension of \cref{prop:TwoLayer_Shrinkage}.

\begin{proposition}[Shrinkage dynamics, multilayer]
  \label{prop:MultiLayer_Shrinkage}
  Suppose that
  \begin{align*}
    \norm{p_S}_{\infty} \leq \eta \norm{\theta^*_S - \theta_S}_{\infty},
    \quad
    \norm{q_S + r_S}_{\infty} \leq \ep,
  \end{align*}
  for some constant $\eta \leq 1/8$ and $\ep > 0$.
  Suppose for some $t_0 \geq 0$ that
  \begin{align*}
    \norm{\theta^*_S - \theta_S(t_0)}_{\infty} \leq M.
  \end{align*}
  Then,
  \begin{align}
    \label{eq:MultiLayer_ShrinkageBound}
    \norm{\theta^*_S - \theta_S(t)}_{\infty} \leq M \vee 2 \ep,\quad \forall t \geq t_0.
  \end{align}
  Furthermore,
  \begin{itemize}
    \item If $\eta M \geq \ep$, then
    \begin{align}
      \norm{\theta^*_S - \theta_S(t)}_{\infty} \leq \frac{1}{2} M,\quad \forall t \geq t_0 + \overline{T}^{\mr{half}}(M),
    \end{align}
    where
    \begin{align}
      \overline{T}^{\mr{half}}(M) &= C_{D,1} M^{-\frac{2D+2}{D+2}} + C_{D,2} M^{-1}\max_{k \in S} R_k^{-D} \ln \frac{R_k}{r_k},
    \end{align}
    where $r_k = \min(\lambda_k^{1/2},b_{0}/\sqrt {D})$ and $R_k = \max(\lambda_k^{1/2},b_{0}/\sqrt {D})$.

    \item If $\eta M < \ep$, then for each $k \in S$ such that $\abs{\theta^*_k} \geq 4\ep$, we have
    \begin{align}
      \abs{\theta^*_k - \theta_k(t)} \leq 4 \ep,\quad \forall t \geq t_0 + \overline{T}^{\mr{fin}}_k,
    \end{align}
    where
    \begin{align}
      \overline{T}^{\mr{fin}}_k &\leq C_{D,3} \abs{\theta^*_k}^{-1} R_k^{-D} \ln \frac{R}{r} + C_{D,4} \abs{\theta^*_k}^{-\frac{2D+2}{D+2}} \ln^+ \frac{\abs{\theta^*_k}}{\ep}.
    \end{align}
    and $C_{D,3}$ and $C_{D,4}$ are constants depending only on $D$.
    On the other hand, for all $k \in S$ we have
    \begin{align}
      \label{eq:MultiLayer_ShrinkageBound_FinalSmallSig}
      \abs{\theta^*_k - \theta_k(t)} \leq 2\max(\abs{\theta^*_k},4\ep),\quad \forall t \geq t_0 + C_{D,5} \ep^{-\frac{2D+2}{D+2}}.
    \end{align}
  \end{itemize}
  Here, $C_{D,1},\dots,C_{D,5}$ are absolute constants depending only on $D$.

\end{proposition}
\begin{proof}
  The proof is similar to that of \cref{prop:TwoLayer_Shrinkage},
  where we use \cref{lem:MultiLayer_AppBelowLargeKappa} and \cref{cor:MultiLayer_AppAboveLargeKappa} for the first part,
  and \cref{cor:MultiLayer_AppBoundSmallKappa} and \cref{cor:MultiLayer_Rectracting} for the second part.
\end{proof}

Let us choose $\nu_1 = \ep / \eta = C n^{-1/2+s_1} \sqrt{\ln n}$.
Define  $M_i = 2^{-i} B$ for $i = 0,1,\dots,I+1$ with $I = \lfloor \log_2 (B/\nu_1) \rfloor$.
Then it is clear that $I \leq \log_2 \sqrt {n} + C$.
Since $\norm{\theta^*_S - \theta_S(0)}_{\infty} = \norm{\theta^*_S}_{\infty} \leq B$,
we use \cref{prop:MultiLayer_Shrinkage} iteratively to obtain that
\begin{align*}
  \norm{\theta^*_S - \theta_S(t)}_{\infty} \leq M_{i+1},\quad \forall t \geq \sum_{j=0}^i T^{(j)},
\end{align*}
where
\begin{align*}
  T^{(i)} = C M_i^{-\frac{2D+2}{D+2}} + C M_i^{-1} \max_{k \in S} R_k^{-D} \ln \frac{R_k}{r_k}
  \leq C M_i^{-\frac{2D+2}{D+2}} + C M_i^{-1} b_0^{-D} \ln n,
\end{align*}
where we use $\lambda_k \asymp k^{-\gamma}$ and $\max S \leq C n^{\kappa/2}$ to obtain
\begin{align*}
  \max_{k \in S} R_k^{-D} \ln \frac{R_k}{r_k} \leq C b_0^{-D} \ln n.
\end{align*}
Now, we have
\begin{align*}
  \overline{T} &\coloneqq \sum_{i=0}^I T^{(i)} \leq C \sum_{i=0}^I M_i^{-\frac{2D+2}{D+2}} + C \sum_{i=0}^I M_i^{-1} \ln n \\
  & \leq  C \nu_1^{-1} b_0^{-D} \ln n + C \nu_1^{-\frac{2D+2}{D+2}}\\
  & \leq C n^{\frac{D+1}{D+2}}.
\end{align*}

Finally, when $t \geq \overline{T}$, we have $\eta M_{I+1} < \ep$.
Choosing $\nu_2 = n^{-1/2} \ln n$, we use the second part of \cref{prop:MultiLayer_Shrinkage} to obtain that
for all $k \in J_{\mr{sig}}(\nu_2)$,
\begin{align*}
  \abs{\theta_k^* - \theta_k(t)} \leq 4 \ep,\quad \forall t \geq \overline{T} + T^{\mr{fin}},
\end{align*}
where
\begin{align*}
  T^{\mr{fin}} &\leq C \max_{k \in J_{\mr{sig}}(\nu_2)} \zk{ \abs{\theta^*_k}^{-1} R_k^{-D} \ln \frac{R}{r} +  \abs{\theta^*_k}^{-\frac{2D+2}{D+2}} \ln^+ \frac{\abs{\theta^*_k}}{\ep}} \\
  &\leq C \nu_2^{-1} b_0^{-D} + C \nu_2^{-\frac{2D+2}{D+2}} \ln n \\
  & \leq C(c_b^{-1} + 1) n^{\frac{D+1}{D+2}}
\end{align*}
On the other hand, for $k \in S \backslash J_{\mr{sig}}(\nu_2)$, we use \cref{eq:MultiLayer_ShrinkageBound_FinalSmallSig} to get
\begin{align*}
  \abs{\theta^*_k - \theta_k(t)} \leq 2\max(\abs{\theta^*_k},4\ep)
\end{align*}
provided $t \geq \overline{T} + C \ep^{-\frac{2D+2}{D+2}}$.
Similar to the argument for the two-layer case,
\begin{align*}
    \sum_{k \in S \backslash J_{\mr{sig}}(\nu_2)} \abs{\theta^*_k - \theta_k(t)}^2
    &= \sum_{k \in S_1 \backslash J_{\mr{sig}}(\nu_2)} \abs{\theta^*_k - \theta_k(t)}^2 + \sum_{k \in S_2 \backslash J_{\mr{sig}}(\nu_2)} \abs{\theta^*_k - \theta_k(t)}^2\\
    & \leq C \sum_{k \in S_1 \backslash J_{\mr{sig}}(\nu_2)} \abs{\theta_k^*}^2
    + C \sum_{k \in S_2 \backslash (S_1 \cup J_{\mr{sig}}(\nu_2))} \ep^2 \\
    & \leq C \Psi(\nu_2) + C n^{\frac{1+s'}{(D+2)\gamma}} \ep^2
\end{align*}

Consequently, we conclude that after $t \geq \underline{C} n^{\frac{D+1}{D+2}} = C(c_b^{-1} + 1) n^{\frac{D+1}{D+2}}$, we have
\begin{align*}
  \norm{\theta^*_S - \theta_S(t)}_2^2
  &\leq C \frac{\ln n}{n} \zk{\Phi\xkm{n^{-1/2}\sqrt {\ln n}} + n^{\frac{1+s'}{(D+2)\gamma}}} + C \Psi\xkm{n^{-1/2}\ln n}.
\end{align*}

\subsubsection{Bounding the Noise}

We will use the second part of \cref{lem:MultiLayer_ErrorControl} to prove the claim that the time $T^{\mr{err}}$ defined in \cref{eq:ErrorTimeMulti}
satisfies $T^{\mr{err}} \geq \bar{C} n^{\frac{D+1}{D+2}}$ for some $\bar{C}$ if $E$ is chosen appropriately.

To apply the lemma, first we have to show \cref{eq:MultiLayer_ErrorControl_Condition} for all $k \in R$, namely
\begin{align}
  \label{eq:Proof_Multi_Error_Bound_Req}
  2^{\frac{D+1}{2}} b_0^D \int_{0}^{t} \xk{\abs{\theta_k^*} + \abs{h_k(s)}} \dd s \leq \ln \frac{b_0/\sqrt {D}}{\lambda_k^{1/2}},
\end{align}
where $h_k(s) = p_k(s) + q_k(s) + r_k(s)$.

Let us denote $T^* = \bar{C} n^{\frac{D+1}{D+2}}$ for some constant $\bar{C}$.
The choice of $b_0$ gives
\begin{align}
  \label{eq:Proof_Multi_TimeTS}
  b_0^D T^* n^{-1/2} \leq \bar{C} C_b.
\end{align}
First, we have
\begin{align*}
  b_0^D \int_{0}^{T^*} \abs{\theta_k^*} \dd s = b_0^D T^* \abs{\theta_k^*} \leq b_0^D T^* n^{-1/2}\sqrt {\ln n} \leq \bar{C} C_b \sqrt {\ln n}.
\end{align*}
Let us denote $\overline{T}^{(i)} = \sum_{j=0}^i \overline{T}^{(j)}$.
Now, for $t \in [\overline{T}^{(i)}, \overline{T}^{(i+1)}]$, using the results in the previous part, we have
$\norm{\theta^*_S - \theta_S(t)}_{\infty} \leq M_i$, so the first bound in \cref{eq:Proof_BasicControls2} gives
\begin{align*}
  \abs{p_k(t)} \leq C M_i n^{-s} \sqrt {\ln k},\quad t \in [\overline{T}^{(i)}, \overline{T}^{(i+1)}].
\end{align*}
Thus,
\begin{align*}
  \int_{0}^{\overline{T}} \abs{p_k(t)}\dd t &\leq C \sum_{i=0}^I M_i n^{-s} \sqrt {\ln k} \cdot T^{(i)} \\
  &\leq C \sum_{i=0}^I M_i n^{-s} \sqrt {\ln k} \cdot \xk{M_i^{-\frac{2D+2}{D+2}} +M_i^{-1} b_0^{-D} \ln n } \\
  &\lesssim \xk{\sum_{i=0}^I M_i^{-\frac{D}{D+2}} n^{-s}  \sqrt {\ln k} + \sum_{i=0}^I b_0^{-D} n^{-s} \sqrt {\ln k}} \\
  & \lesssim v_1^{-\frac{D}{D+2}} + b_0^{-D} n^{-s} \sqrt {\ln k}  \log \nu_1^{-1} \\
  & \lesssim v_1^{-\frac{D}{D+2}} + b_0^{-D} \sqrt {\ln k},
\end{align*}
and
\begin{align*}
  b_0^D \int_{0}^{\overline{T}} \abs{p_k(t)}\dd t \lesssim C_b \sqrt{\ln k}.
\end{align*}
Moreover, for $t \geq \overline{T}$, we have
\begin{align*}
  \abs{p_k(t)} \leq \eta_k \norm{(\theta^* - \theta)_S}_{\infty} \leq C n^{-s} \sqrt {\ln (nk)} \cdot M_{I+1}
  \leq C n^{-1/2} \sqrt {\ln k},
\end{align*}
so combined with \cref{eq:Proof_Multi_TimeTS},
\begin{align*}
  b_0^D \int_{\overline{T}}^{T^*} \abs{p_k(t)}\dd t \leq C \bar{C} C_b \sqrt {\ln k}.
\end{align*}
Moreover, the other two bounds in \cref{eq:Proof_BasicControls2} gives
\begin{align*}
  \abs{q_k + r_k} \leq C n^{-1/2} \sqrt {\ln n + \ln k},
\end{align*}
so
\begin{align*}
  b_0^D \int_{0}^{T^*} \abs{q_k(t) + r_k(t)}\dd t \leq \bar{C} C_b \sqrt {\ln n + \ln k}.
\end{align*}
Combining all these terms, we get
\begin{align}
  \label{eq:Proof_Multi_ErrorBound_1}
  2^{\frac{D+1}{2}} b_0^D \int_{0}^{T^*} \xk{\abs{\theta^*_k} + \abs{h_k(t)}}\dd t
  \leq C \bar{C} C_b \xk{\sqrt {\ln n} + \sqrt {\ln k}}.
\end{align}

Now, by our choice of $R = S^\complement$, we have $\ln \lambda_k^{-1} \geq \frac{1+s'}{D+2}\ln n$ (also $k \geq L \asymp n^{\frac{1+s'}{(D+2)\gamma}}$), and thus
\begin{align*}
  \ln \frac{b_0/\sqrt {D}}{\lambda_k^{1/2}}
  &\geq -\frac{1}{2(D+2)}\ln n + \frac{1}{2}\ln \lambda_k^{-1} - C\\
  &\geq -\frac{1}{2(D+2)}\ln n + \frac{1}{2} \frac{1+s'/2}{1+s'}\ln \lambda_k^{-1} + \frac{1}{2}\frac{s'/2}{1+s'} \ln \lambda_k^{-1} - C \\
  & \geq \frac{s'}{4(D+2)}\ln n + \frac{s' \gamma }{4(1+s')} \ln k - C
\end{align*}
Therefore, as long as $n$ is sufficiently large, \cref{eq:Proof_Multi_Error_Bound_Req} holds for all $k \in R$.
Now, plugging \cref{eq:Proof_Multi_ErrorBound_1} into the upper bound \cref{eq:MultiLayer_ErrorControl} in \cref{lem:MultiLayer_ErrorControl},
we prove the claim.

\subsubsection{Generalization Error}

Now, we can fix $c_{\mr{b}}, C_{\mr{b}}, E$ and take them as constants.
Now, when $t \in [\underline{C} n^{\frac{D+1}{D+2}}, \bar{C} n^{\frac{D+1}{D+2}}]$,
the error term is bounded by
\begin{align*}
  \norm{\theta_R(t)}_2^2 & \leq C \sum_{k \in R} \lambda_k^2 b_0^{2D} \exp(C \bar{C} \xk{\sqrt {\ln n} + \sqrt {\ln k}}) \\
  & \leq C n^{-\frac{D}{D+2}} \exp(C\bar{C} \sqrt {\ln n}) \sum_{k \in R} \lambda_k^2 \exp(2C \sqrt {\ln k})\\
  & \leq C n^{-\frac{D}{D+2}}\exp(C\bar{C} \sqrt {\ln n})  \sum_{k \geq L} \lambda_k^2 \exp(2C \sqrt {\ln k}) \\
  & \leq C n^{-\frac{D}{D+2}} \exp(C\bar{C} \sqrt {\ln n}) L^{1-2\gamma+s} \\
  & \leq C n^{-1+\frac{1+s}{(D+2)\gamma}}.
\end{align*}
Consequently,
\begin{align*}
  \norm{\theta^*_R - \theta_R(t)}_2^2
  \leq 2 \norm{\theta^*_R}_2^2 + 2 \norm{\theta_R(t)}_2^2
  \leq \Psi(n^{-1/2}\sqrt {\ln n}) + C n^{-1+\frac{1+s}{(D+2)\gamma}}.
\end{align*}
Finally, the result follows by combining the signal and noise components.



\section{Other proofs}

\subsection{Proof of \cref{prop:EigLearn}}

We first consider the signal part.
From the proof of \cref{thm:EigenvalueGD} or \cref{thm:EigenvalueDeepGD}, for the final choice of stopping time $t$,
we have
\begin{align*}
  \abs{\theta_k^* - \theta_k(t)} \leq C \ep = C n^{-1/2} \sqrt {\ln n},\quad \forall k \in J_{\mr{sig}}(n^{-\hf} \ln n).
\end{align*}
Therefore, if $\abs{\theta_k^*} \geq C_0 n^{-1/2} \ln n$ for some large constant $C_0$,
we have $\abs{\theta_k(t)} \geq \frac{1}{2} \abs{\theta_k^*}$.
Now, we analyze the one-dimensional dynamics.
For $D=0$, we use \cref{eq:EqTwoLayerP_Conservation} to get
\begin{align*}
  \abs{\theta(t)} = a(t) \abs{\beta(t)} = a(t) \sqrt {a^2(t) - \lambda} \leq a^2(t),
\end{align*}
while for $D > 0$, we apply \cref{eq:MultiLayer_Conservation} to get
\begin{align*}
  \abs{\beta(t)} = \sqrt {a^2(t) - \lambda} \leq a(t),\quad \abs{\beta(t)}=\sqrt {(b^2(t) - b_0^2)/D}\leq b(t).
\end{align*}
and thus
\begin{align*}
  \abs{\theta(t)} = a(t)b^D(t) \abs{\beta(t)} \leq a(t)b^D(t) a^{\frac{1}{D+1}}(t) b^{\frac{D}{D+1}}(t)
  = \zk{a(t) b^D(t)}^{\frac{D+2}{D+1}}.
\end{align*}
Combining the above two cases, we have
\begin{align*}
  \zk{a(t) b^D(t)}^{\frac{D+2}{D+1}}
  \geq \abs{\theta_k(t)} \geq \frac{1}{2} \abs{\theta_k^*},\quad
  \Longrightarrow \quad a(t) b^D(t) \geq c \abs{\theta_k^*}^{\frac{D+1}{D+2}}.
\end{align*}

For the noise part, we have to inspect the ``bounding the noise'' part in the proof of the theorem where we apply \cref{lem:TwoLayer_ErrorControl} (or \cref{lem:MultiLayer_ErrorControl} for the deeper parameterization).
Still, it suffices to focus on the one-dimensional dynamics.
Instead of bounding $\theta$, we use the bound of $\beta$ in \cref{eq:TwoLayer_ErrorControl_Beta} (or \cref{eq:MultiLayer_ErrorControl_Beta})
to get
\begin{align*}
  \abs{a_k(t)b_k^D(t)} 
  &\leq 2^{\frac{D+1}{2}} b_0^D \abs{\beta}
  \leq 2^{\frac{D+1}{2}} \lambda_k^{\hf}  b_0^D \exp(2^{\frac{D+1}{2}} b_0^D \int_{0}^t \xk{\abs{\theta^*_k} + \abs{h_k(\tau)}} \dd \tau) \\
  & = C a_k(0) b_k^D(0) \exp(2^{\frac{D+1}{2}} b_0^D \int_{0}^t \xk{\abs{\theta^*_k} + \abs{h_k(\tau)}} \dd \tau)\\
  & \leq  C a_k(0) b_k^D(0) \exp(C (\sqrt{\ln n} +\sqrt{\ln k})).
\end{align*}


\section{Auxiliary Results}

\subsection{Concentration Inequalities}

The following is the standard Hoeffding's inequality for sub-Gaussian random variables, see, e.g., \citet[Theorem 2.6.2]{vershynin2018_HighdimensionalProbability}.
\begin{lemma}
  \label{lem:HoeffdingSubG}
  Let $X, X_1,\dots,X_n$ be i.i.d.\ $K$-sub-Gaussian random variables.
  Then, for every $t \geq 0$, we have
  \begin{align*}
    \bbP \dk{ \abs{\frac{1}{n}\sum_{i=1}^n X_i - \E X} \geq t} \leq 2 \exp \xk{-\frac{cnt^2}{K^2}},
  \end{align*}
  where $c > 0$ is an absolute constant.
\end{lemma}

The following lemma gives a Hilbert space version of Hoeffding's inequality, see, e.g., \citet[Theorem 3]{pinelis1986remarks}.
\begin{lemma}
  \label{lem:HoeffdingHilbert}
  Let $H$ be a separable Hilbert space.
  Let $X, X_1,\dots,X_n$ are i.i.d.\ random variables taking values in $H$ with $\norm{X_i} \leq B$ almost surely.
  Then, for every $t\geq 0$,
  \begin{align}
    \bbP \dk{\norm{\frac{1}{n}\sum_{i=1}^n X_i - \E X} \geq t} \leq 2 \exp \xk{-\frac{nt^2}{2B^2}}.
  \end{align}
\end{lemma}



\subsection{Results on ordinary differential equations}

\begin{proposition}
  \label{prop:ODE_Power}
  Let $k > 0, p > 1$.
  Consider the ODE $\dot{x} \geq k x^p$, $x(0) = x_0 > 0$,
  we have
  \begin{align*}
      x(t) \geq \xk{x_0^{-(p-1)} - (p-1) kt}^{-\frac{1}{p-1}}, \qq{and}
       \inf \dk{t \geq 0: x(t)\geq M} \leq \zk{(p-1)k x_0^{p-1}}^{-1},\quad M \geq 0.
    \end{align*}
  Similarly, for the ODE $\dot{x} \leq - k x^p$, $x(0) = x_0 > 0$,
  we have
  \begin{align*}
      & x(t) \leq \xk{x_0^{-(p-1)} + (p-1) kt}^{-\frac{1}{p-1}}, \qq{and}
       \inf \dk{t \geq 0: x(t)\leq M} \leq \zk{(p-1)k M^{p-1}}^{-1},\quad  M > 0.
    \end{align*}
%
%

\end{proposition}


\bibliography{main}

\end{document}